\documentclass[a4paper,11pt]{amsart}

\usepackage[margin=0.95in]{geometry}

\usepackage[utf8]{inputenc}

\usepackage{amsmath}
\usepackage{amssymb}
\usepackage{xcolor}
\usepackage{bbm}
\usepackage{graphicx}
\usepackage{verbatim}
\usepackage{algorithm2e}
\usepackage{nicefrac}
\usepackage{todonotes}
\usepackage[T1]{fontenc}
\usepackage[german, french,USenglish]{babel}
\usepackage{comment}

\makeatletter

\theoremstyle{plain}
\newtheorem{theorem}{Theorem}
\newtheorem{thm}[theorem]{Theorem}

\newtheorem{proposition}[theorem]{Proposition}
\newtheorem{lemma}[theorem]{Lemma}
\newtheorem{cor}[theorem]{Corollary}
\newtheorem{corollary}[theorem]{Corollary}

\newtheorem{assump}[theorem]{Assumption}

\theoremstyle{remark}

\newtheorem{remark}[theorem]{Remark}

\newcommand{\R}{\mathbb{R}}

\newcommand{\E}{\mathbb{E}}

\renewcommand{\epsilon}{\varepsilon}

\newcommand{\essinf}{\mathop{\mathrm{ess\,inf}}\displaylimits}

\def\JELname{{\bfseries JEL Classification}\enspace}
      \def\JEL#1{\par\addvspace\medskipamount{\rightskip=0pt plus1cm
      \def\and{\ifhmode\unskip\nobreak\fi\ $\cdot$
      }\noindent\JELname\ignorespaces#1\par}}

\makeatletter
\newcommand{\leqnomode}{\tagsleft@true\let\veqno\@@leqno}
\newcommand{\reqnomode}{\tagsleft@false\let\veqno\@@eqno}
\makeatother

\begin{document}
\title[Schr\"{o}dinger bridge problem]
{%Polynomial time 
Forward Reverse Kernel  Regression  for the Schr\"{o}dinger bridge problem }
\author[D.~Belomestny]{Denis Belomestny$^{1}$}
\address{$^1$Faculty of Mathematics\\
Duisburg-Essen University\\
Thea-Leymann-Str.~9\\
D-45127 Essen\\
Germany}
%\address{$^2$Faculty of Computer Sciences\\
%HSE University
%\\
%Moscow, Russian Federation}
\email{denis.belomestny@uni-due.de}

\author[J.~Schoenmakers]{John Schoenmakers$^{2}$}
\address{$^2$Weierstrass Institute for Applied Analysis and Stochastics \\
\mbox{Mohrenstr.~39} \\
10117 Berlin \\
Germany}
\email{schoenma@wias-berlin.de}


\date{}

\begin{abstract}
In this paper, we study the Schr\"odinger Bridge Problem (SBP), which is central to entropic optimal transport.
For general reference processes and begin--endpoint distributions, we propose a forward-reverse iterative Monte Carlo procedure to approximate the Schr\"odinger potentials in a nonparametric way. In particular, we use kernel based Monte Carlo regression in the context of Picard iteration of a corresponding fixed point problem as considered in \cite{CGP2016}. By preserving   in the iteration positivity and contractivity in a Hilbert metric sense, we develop a provably convergent algorithm. Furthermore, we provide convergence rates for the potential estimates and prove their optimality. {Finally, as an application, we propose a non-nested  Monte Carlo procedure for the final dimensional distributions of the
Schr\"odinger Bridge process, based on the constructed potentials and the forward-reverse simulation method for conditional diffusions developed in \cite{BS2014}. %%a generative modeling.
}

\end{abstract}

\maketitle

\section{Introduction}\label{intro}
The  Schr\"odinger bridge problem (SBP) traces back to a question of Erwin Schr\"odinger in ~\cite{schrodinger1932}:  
\emph{among all evolutions of a system that start in a prescribed distribution and end in another one, which is the most likely when likelihood is measured by relative entropy with respect to a fixed reference process ?}  
Besides its physical origin, the SBP is now known as an entropic analogue of optimal transport and as a stochastic control problem~\cite{leonard2014,daipra1991}.
Let the reference Markov process run in $\R^{d}$ with transition density
\[
  q(s,x;\,t,z), \qquad 0\le s\le t\le T,\; x,z\in\R^{d}.
\]
The aim is to build a Markov process whose joint start-end law
\begin{equation}\label{eq:schro_factor}
  \mu(dx,dz)=q(0,x;\,T,z)\,\nu_{0}(dx)\,\nu_{T}(dz),
\end{equation}
matches two fixed marginals 
\(
  \mu(dx,\R^{d})=\rho_{0}(x)\,dx,\;
  \mu(\R^{d},dz)=\rho_{T}(z)\,dz .
\)
The unknown measures $\nu_{0},\nu_{T}$ are the \emph{boundary
potentials}.  Existence of such potentials was proved by Fortet in one
dimension~\cite{Fortet}, by Beurling in any dimension~\cite{Beur}, and
revisited through a Banach fixed-point argument in~\cite{CGP2016}.  A
recent extension to non-compact supports is given in~\cite{eckstein2025hilbert}.  
Whenever the factorization~\eqref{eq:schro_factor} holds, there exists a
\emph{Schr\"odinger Markov process} $X^{\mu}$ such that for any grid
$0<t_{1}<\dots<t_{n}<T$ and every bounded Borel function $g$ on $\mathbb{R}^{d(n+2)}$,
\begin{equation}\label{eq:cond_diffusion}
  \E\!\bigl[
       g\bigl(X_{0}^{\mu},X_{t_{1}}^{\mu},\dots,X_{t_{n}}^{\mu},X_{T}^{\mu}\bigr)
     \bigr]
  \;=\!
  \int_{\R^{d}\times\R^{d}}
    \mu(dx,dz)\,
    \E\!\bigl[
      g\bigl(x,X_{t_{1}}^{x},\dots,X_{t_{n}}^{x},z\bigr)\,
      \big|\,X_{T}^{x}=z
    \bigr]
\end{equation}
where $X^{x}$ denotes the reference process started in $x$ at time $0$.
For an \emph{arbitrary} coupling~$\mu$, the right-hand side of
\eqref{eq:cond_diffusion} still determines a \emph{reciprocal
process}\footnote{See Appendix~\ref{sec:rp} for a short review.}, which
possesses only a weak (two-time) Markov property.  Jamison~\cite{Jam74} proved that
this process is genuinely Markov if and only if $\mu$ factorizes as in
\eqref{eq:schro_factor}.  In other words, the factorization
criterion characterises precisely when a reciprocal family can be
promoted to a Markov one - the hallmark of a Schr\"odinger bridge.
\par
 In many presentations of the Schr\"odinger Bridge problem, one takes a very simple reference process, for instance some Brownian motion with drift, so that its transition kernel is explicitly known, see e.g. \cite{pooladian2024plug} and \cite{baptista2024conditional}. However, there are several practical and theoretical advantages in considering more general reference processes, for example processes given by multidimensional Stochastic Differential Equations (SDEs) possibly restricted  to certain domain constraints. 
 Loosely speaking,
 in the SBP the new process is found by reweighting the paths of the reference process to satisfy the desired endpoint distributions. If the reference is already ``close'' (in distribution sense) to the target boundary marginals, the amount of correction required is ``small'', and  iterative numerical procedures may converge  rapidly.  
\par
In principle, \(\nu_{0}\) and \(\nu_{T}\) can be computed by a forward--backward iteration scheme analog to the Sinkhorn or IPFP scheme in discrete entropic optimal transport.  In particular, if \(q\) is the transition kernel of the reference process and \(\rho_{0}\) and \(\rho_{T}\) are prescribed boundary densities, then according to (\ref{eq:schro_factor}), \(\nu_{0}\) and \(\nu_{T}\) satisfy
\begin{align}
\rho_{0}(x) \;&=\; \nu_{0}(x)\,\int q(0,x;T,z)\,\nu_{T}(z)\,dz,\notag \\
\rho_{T}(z) \;&=\; \nu_{T}(z)\,\int q(0,x;T,z)\,\nu_{0}(x)\,dx,\label{eq:se-intr}
\end{align}
Conceptually, one can attempt a Picard (fixed-point) iteration:
\begin{equation}\label{picit}
\nu_{T}^{(n)} \;\longrightarrow\; \nu_{0}^{(n+1)}, 
\quad
\nu_{0}^{(n+1)}\;\longrightarrow\; \nu_{T}^{(n+1)},
\end{equation}
thus updating each potential function estimate based on an estimate of the other one, until convergence within some prescribed accuracy level is achieved.  When the reference process and its transition densities are well understood (e.g., known analytically and  low- dimensional), it is possible to discretize and solve these integral equations directly.  However, this becomes computationally (too) challenging  for more complex reference processes, in particular in higher dimensions.
\par

An attractive alternative is to solve the Schr\"{o}dinger system
(\ref{eq:se-intr}) \emph{stochastically}, using Monte Carlo approximations of
the involved integrals. More precisely, observe that system \eqref{eq:se-intr}
can be written in stochastic terms:
\begin{align}
\rho_{0}(x) &  \;=\;\nu_{0}(x)\,\mathbb{E}\bigl[\nu
_{T}(X_{T}^{x})\bigr],\label{eq:schr-exp-1}\\
\rho_{T}(z) &  \;=\;\nu_{T}(z)\,\mathbb{E}\bigl[\,\nu_{0}(Y_{T}^{z}%
)\,\mathcal{Y}_{T}^{z}\bigr],\label{eq:schr-exp-2}
\end{align}
where $X^{x}$ is the \textquotedblleft forward\textquotedblright\ reference
process starting in $x$ at time $0$ and $(Y^{z},\mathcal{Y}^{z})$ is a
suitably chosen \textquotedblleft reverse\textquotedblright\ process running
through $\mathbb{R}^{d}\times\mathbb{R}_{+}$ with $(Y_{0}^{z},\mathcal{Y}%
_{0}^{z})=(z,1)$. Let us underline that we do not need to assume an explicit
closed-form for the reference transition kernel $q(s,x;t,y)$. Instead, we only
require the ability to sample from the forward process $X$ with the same
dynamics as the reference, and to sample from the \textquotedblleft
reverse\textquotedblright\ process $(Y,\mathcal{Y})$. The construction of the
reverse process goes back to \cite{Thomson1987} for special cases that allow
for $\mathcal{Y}\equiv1$. A generalization to general diffusions was
%rigorously
constructed in \cite{MSS2004}. The details are spelled out in
Appendix~\ref{subsec:reverse_diffusion_construction}.

Having at hand some estimate $\nu_{T}^{(n)}$ say, one may carry out the
updates in (\ref{picit}) via the following Monte Carlo kernel regression
procedure (rough sketch): We construct

\begin{itemize}
\item forward paths $\{(X_{0}^{(i)},X_{T}^{(i)})\}_{i=1}^{N}$ under some
initial distribution of $X_{0}$,

\item reverse paths $\{(Y_{0}^{(j)},Y_{T}^{(j)},\mathcal{Y}_{T}^{(j)}%
)\}_{j=1}^{M}$ under some initial distribution of $Y_{0}$.
\end{itemize}

Next, using kernel (Nadaraya--Watson) regression, we approximate the
conditional expectation in (\ref{eq:schr-exp-1}) by
\[
\widehat{g}(x)\;=\;\frac{\sum_{i=1}^{N}K\left(  (x-X_{0}^{(i)})/\delta\right)
\,\nu_{T}^{(n)}(X_{T}^{(i)})}{\sum_{i=1}^{N}K\left(  (x-X_{0}^{(i)}%
)/\delta\right)  },
\]
where $K$ is a suitable kernel, and set $\nu_{0}^{(n+1)}(x)\;=\;\rho
_{0}(x)/\widehat{g}(x).$ Similarly, we estimate the expectation in
(\ref{eq:schr-exp-2}) by%
\[
\widehat{h}(x)\;=\;\frac{\sum_{j=1}^{M}K\left(  (x-Y_{0}^{(j)})/\delta\right)
\,\nu_{0}^{(n+1)}(Y_{T}^{(j)})\mathcal{Y}_{T}^{(j)}}{\sum_{j=1}^{M}K\left(
(x-Y_{0}^{(j)})/\delta\right)  \mathcal{Y}_{T}^{(j)}},
\]
and set $\nu_{T}^{(n+1)}(x)\;=\;\rho_{T}(x)/\widehat{h}(x)$.
\par
The above updating procedure may be repeated until  no  improvement  
within a certain accuracy level is obtained any more.   
We thus obtain a continuous approximation to the boundary potentials \(\nu_{0}\) and \(\nu_{T}\) in a flexible, data-driven manner. In particular, this method provides an explicit functional representation of the Schr\"odinger boundary potentials in settings where classical deterministic methods are prohibitive, bridging the gap between rigorous entropic optimal transport theory and practical high-dimensional Monte Carlo implementations. Furthermore, we analyze the convergence of the proposed iteration scheme.  A cornerstone of our analysis is the fact that the forward--reverse Monte Carlo iteration 
remains a \emph{contraction} in the Hilbert projective metric, which is recapitulated in Appendix~\ref{recapHM}. Next as an application we analyze the problem of generating Schr\"odinger Bridge using $h$-transform techniques for SDEs  based on the estimated potentials. 
Finally, as another application, we show that the finite dimensional distributions 
of a Schr\"odinger Bridge process can be estimated by a non-nested Monte Carlo procedure if the potentials $\nu_0$ and $\nu_T$ are given, or constructed by the (likewise non-nested) Monte Carlo procedure presented in this paper. In fact, this is achieved via an application of the forward-reverse 
simulation procedure for conditional diffusions developed in \cite{BS2014}.

Theoretical and numerical analysis of the SBP and its iterative solution has been extensively studied. L\'eonard~\cite{leonard2013survey} provides a foundational overview of the Schr\"odinger problem, its entropy minimization formulation, and convergence properties. Chen, Georgiou, and Pavon~ \cite{chen2016entropic} interpret the SBP in terms of stochastic control and analyze the convergence of the iterative scaling algorithm. Peyr\'e and Cuturi~ \cite{peyre2019computational} frame the problem within entropic optimal transport and demonstrate numerical schemes based on Sinkhorn iteration. Benamou et al.~\cite{benamou2015iterative} introduce iterative Bregman projections, a generalization of Picard iteration for entropy-regularized problems. Cominetti, Soto, and R\'ios~\cite{cominetti2021rate} analyze the convergence rate of Sinkhorn-like iterations. De Bortoli et al. ~\cite{debortoli2021neurips} explore neural approaches that learn Schr\"odinger potentials using iterative schemes as part of model training. Pavon, Tabak, and Trigila~\cite{pavon2018data} propose an iterative method for solving the Schr\"odinger bridge problem when the marginals are only known via samples. Their approach generalizes Fortet--Sinkhorn iterations by combining importance sampling and constrained maximum likelihood estimation to propagate the Schr\"odinger potentials. This sample-based method is particularly well-suited for high-dimensional applications, where grid-based methods become infeasible. 
\par
In contrast to classical approaches relying on analytic forms of the transition density of the reference process ~\cite{leonard2013survey, chen2016entropic, benamou2015iterative}, our paper offers a nonparametric and data-driven framework in the case of general reference processes for which only a generative model is available. We develop a kernel-based estimation methodology  that allows for efficient estimation of Schr\"odinger potentials using forward and reverse samples from the reference process. Furthermore, we establish strong theoretical guarantees for the convergence and performance of the proposed method. Notably, we derive for the first time in the literature minimax-optimal rates of functional approximation for the Schr\"odinger potentials in Hilbert's metric based  on samples from the reference process.
\par
The paper is organized as follows. First, we review the Schr\"odinger Bridge problem and introduce some notations in Section~\ref{alg:iterative_kernel}. Section~\ref{sec:aip} is devoted to the description of our iterative  kernel  regression algorithm. In Section~\ref{sec:conv}, we present our convergence analysis.
A short perturbation analysis of the actual trajectories of the SB due to approximated potentials is done in Section~\ref{SBsim}.
In Section~\ref{FRSP} we outline a non-nested simulation procedure for the finite dimensional distributions of a  Schr\"odinger Bridge process. 
The proofs of our convergence results are deferred to Section~\ref{proofs}. Appendices~A--F  
recapitulate the for our goals relevant concepts and results from the literature.

\section{Schr\"{o}dinger problem as a fixed point problem}\label{alg:iterative_kernel}

In this section, we present the essentials of the Schr\"odinger system 
(\ref{eq:schro_factor})
following  \cite{CGP2016}.

\begin{theorem}{ \cite[Prop.~1]{CGP2016}}
\label{pavon}  Let $q(0,\cdot;T,\cdot)$ be continuous and strictly positive on
$\mathbb{R}^d\times\mathbb{R}^d$. Then for
given densities $\rho_{0}$ and $\rho_{T}$ with compact supports $\mathsf{S}_{0}%
\subset\mathbb{R}^{d}$\ and $\mathsf{S}_{T}\subset\mathbb{R}^{d},$ respectively, there
exist Borel measurable functions $\nu_{0}:\mathsf{S}_{0}\rightarrow\mathbb{R}_{\geq0}$ and
$\nu_{T}:\mathsf{S}_{T}\rightarrow\mathbb{R}_{\geq0}$ such that
\begin{align}
\rho_{0}(x) &  =\nu_{0}(x)\int_{\mathsf{S}_{T}}q(0,x;T,z)\nu_{T}(z)\,dz,\text{
a.e. on }\mathsf{S}_{0},\nonumber\\
\rho_{T}(z) &  =\nu_{T}(z)\int_{\mathsf{S}_{0}}q(0,x;T,z)\nu_{0}(x)\,dx\text{
a.e. on }\mathsf{S}_{T}.\label{SPT}%
\end{align}
Moreover, if $\nu_{0}^{\prime}$ and $\nu_{T}^{\prime}$ is another pair of solutions, one
has that $\nu_{0}^{\prime}=c\nu_{0}$ and $\nu_{T}^{\prime}=c^{-1}\nu_{T}$ for
some $c>0.$
\end{theorem}

\begin{corollary}
If $\nu_{0}$ and $\nu_{T}$ are as in Theorem~\ref{pavon}, then there is a reciprocal  Markov process \(X^{\mu}\) with 
finite dimensional distributions \eqref{eq:cond_diffusion},  where    \(\mu\left(  dx,dz\right)  =q(0,x;T,z)\nu_{0}(x)\nu_{T}(z)\,dx\,dz.\)
\end{corollary}
Theorem~\ref{pavon} can be proved by establishing the contraction of an
operator $\mathcal{C}$ defined as
\begin{equation}
\mathcal{C}[g]=\int_{\mathsf{S}_{0}}\frac{\rho_{0}(x)}{\int_{\mathsf{S}_{T}%
}q(0,x;T,z)\frac{\rho_{T}(z)}{g(z)}\,dz}q(0,x;T,\cdot)\,dx\label{Cop}%
\end{equation}
in the Hilbert metric (Appendix~\ref{recapHM}). The operator $\mathcal{C}[g]$ is essentially a
composition of positive linear integral transforms (with strictly positive
kernels) and pointwise reciprocals of functions. These operations preserve
positivity, so $\mathcal{C}[g]$ remains strictly positive whenever $g$ is.
Moreover, $\mathcal{C}$ is positively homogeneous, meaning that scaling $g$ by
a positive constant does not affect the \textquotedblleft
core\textquotedblright\ of the map. This is precisely why the Hilbert distance
$d_{H}(f,g)$ between two strictly positive functions $f$ and $g$ is the
natural choice here: it is invariant under scalings of $f$ and $g$ with
arbitrary positive scaling factors (for further details see Appendix~\ref{recapHM}). Then Birkhoff's theorem essentially
implies that such compositions of strictly positive integral operators and
reciprocal maps become strict contractions in the Hilbert metric on the cone%
\[
\mathcal{L}_{+}^{\infty}(\mathsf{S}_{T}):={\displaystyle\bigcup
\limits_{\epsilon>0}}\mathcal{L}_{\epsilon}^{\infty}(\mathsf{S}_{T})\text{
with }\mathcal{L}_{\epsilon}^{\infty}(\mathsf{S}_{T}):=\left\{  f\in
\mathcal{L}^{\infty}(\mathsf{S}_{T}):f(x)\geq\epsilon\text{ \ for a.e. }%
x\in\mathsf{S}_{T}\right\}
\]
under suitable irreducibility assumptions. For details and a historical
overview see \cite{LemNus13} for example. In \cite{CGP2016} it is shown that
for $\mathcal{C}$ given by (\ref{Cop}), under the conditions of
Theorem~\ref{pavon}, there is a constant $\kappa=\kappa(\mathcal{C})<1$ such
that
\[
d_{H}(\mathcal{C}[f],\mathcal{C}[g])\,\leq\,\kappa\,d_{H}(f,g),
\]
for all strictly positive $f,g$ $\in$ $\mathcal{L}_{+}^{\infty}(\mathsf{S}_{T})$. So by the usual fixed-point argument (adapted
to the metric $d_{H}$ that ignores scalar multiples), $\mathcal{C}$ has a
unique fixed point (up to scaling) that satisfies $\mathcal{C}\left(
g^{\star}\right)  =g^{\star}$ in $d_{H}$ sense, that is $\mathcal{C}\left(
g^{\star}\right)  =\alpha g^{\star}$ almost everywhere for some $\alpha>0.$ In
\cite{CGP2016} it is moreover shown that  $\alpha=1$, and that $g^{\star}$ is continuous, i.e. has a  version in  $\mathcal{L}_{+}^{\infty}(\mathsf{S}_{T})$
that is continuous on the whole $\mathsf{S}_{T}$. 
Then given this $g^{\star
}$ satisfying $\mathcal{C}\left(  g^{\star}\right)  =g^{\star}$ almost
everywhere, the solution in Theorem~\ref{pavon} is obviously determined by%
\begin{equation}
\nu_{T}=\frac{\rho_{T}}{g^{\star}}\text{ \ \ and \ }\nu_{0}=\frac{\rho_{0}%
}{\int_{\mathsf{S}_{T}}q(0,\cdot;T,z)\nu_{T}(z)dz}.\label{fp1}%
\end{equation}

Let us separately consider three degenerate cases.
\begin{enumerate}
 \item Suppose that $\mathsf{S}_{0}=\left\{  x_{0}\right\}  ,$ for some $x_{0}\in
\mathbb{R}^{d},$ i.e. $\rho_{0}=\delta(\cdot-x_{0}).$ Then $g^{\ast
}=\mathcal{C}\left(  g^{\star}\right)  $ formally implies%
\[
g^{\star}=\frac{q(0,x_{0};T,\cdot)}{\int_{\mathsf{S}_{T}}q(0,x_{0};T,z)\frac{\rho
_{T}(z)}{g^{\star}(z)}dz}=cq(0,x_{0},1,\cdot)
\]
for some $c>0,$ and then (\ref{fp1}) yields
\[
\nu_{T}=c^{-1}q(0,x_{0};1,\cdot)^{-1}\rho_{T},\text{ \ \ and \ \ }\nu
_{0}=c\delta(\cdot-x_{0}).
\]

\item Similarly, suppose that $\mathsf{S}_{T}=\left\{  z_{0}\right\}  ,$ for some
$z_{0}\in\mathbb{R}^{d},$ i.e. $\rho_{T}=\delta(\cdot-z_{0}).$ Then $g^{\ast
}=\mathcal{C}\left(  g^{\star}\right)  $ formally implies%
\[
g^{\star}=g^{\star}(z_{0})\int_{S_{0}}\frac{q(0,x;T,\cdot)}{q(0,x;T,z_{0})}%
\rho_{0}(x)dx,
\]
and then (\ref{fp1}) yields%
\[
\nu_{T}=c\delta(\cdot-z_{0})\text{\ \ and \ \ }\nu_{0}=c^{-1}q(0,\cdot
;T,z_{0})^{-1}\rho_{0}%
\]
with $c=g^{\star}(z_{0})^{-1}$ can be taken arbitrarily.

\item If both start and end point distribution are degenerated, we thus have
the classical bridge and get $\nu_{T}=c\delta(\cdot-z_{0})$ and $\nu
_{0}=c^{-1}q(0,x_{0},1,z_{0})^{-1}\delta(\cdot-x_{0}),$ where $c>0$ can be
taken arbitrarily. 
\end{enumerate}

So, in the above degenerate cases the Schr\"{o}dinger problem has a relatively
trivial solution. We henceforth assume that both start and end point
distribution are non-degenerated.

\section{Iterative approximation procedure}
\label{sec:aip}
In this section we will spell out in detail an iterative Monte Carlo regression procedure 
as heuristically sketched in Section~\ref{intro}. The procedure yields
an  
approximation to the fixed point $g^\star$ of the operator (\ref{Cop}), and hence  via (\ref{fp1}) an approximation to the potential functions $\nu_0$ and $\nu_T$ due to Theorem~\ref{pavon}.
The operator (\ref{Cop}) may be decomposed as%
\begin{equation}
\mathcal{C}=\mathcal{E}_{0}\circ\mathcal{D}_{0}\circ\mathcal{E}_{T}%
\circ\mathcal{D}_{T}\label{cdec}%
\end{equation}%
 with
\begin{align*}
\mathcal{D}_{0}  &  :\mathcal{L}_{+}^{\infty}(\mathsf{S}_{0})\ni f\rightarrow
1/f\in\mathcal{L}_{+}^{\infty}(\mathsf{S}_{0}),\\
\mathcal{D}_{T}  &  :\mathcal{L}_{+}^{\infty}(\mathsf{S}_{T})\ni f\rightarrow
1/f\in\mathcal{L}_{+}^{\infty}(\mathsf{S}_{T}),\\
\mathcal{E}_{T}  &  :\mathcal{L}_{+}^{\infty}(\mathsf{S}_{T})\ni f\rightarrow
\int_{\mathsf{S}_{T}}q(0,\cdot;T,z)\rho_{T}(z)f(z)\, dz\in\mathcal{L}_{+}^{\infty}%
(\mathsf{S}_{0}),\\
\mathcal{E}_{0}  &  :\mathcal{L}_{+}^{\infty}(\mathsf{S}_{0})\ni f\rightarrow
\int_{\mathsf{S}_{0}}\rho_{0}(x)f(x)q(0,x;T,\cdot)\, dx\in\mathcal{L}_{+}^{\infty}(\mathsf{S}_{T}).
\end{align*}
The operators $\mathcal{E}_{0}$ and $\mathcal{E}_{T}$ have, respectively, the
following stochastic representations 
\begin{align}
\mathcal{E}_{T}[f](x)  &  =\mathbb{E}\left[  \rho_{T}(X_{T}^{x})f(X_{T}%
^{x})\right]  ,\text{ \ \ }f\in\mathcal{L}_{+}^{\infty}(\mathsf{S}_{T}),
\notag \\
\mathcal{E}_{0}[f](z)  &  =\mathbb{E}\left[  \rho_{0}(Y_{T}^{z})f(Y_{T}%
^{z})\mathcal{Y}_{T}^{z}\right]  ,\text{ \ \ }f\in\mathcal{L}_{+}^{\infty
}(\mathsf{S}_{0})\notag \label{e0reverse}
\end{align}
where \((Y,\mathcal{Y})\) is termed a {reverse process}, see Appendix~\ref{subsec:reverse_diffusion_construction} for more details and references on reverse processes in diffusion setting. %By combining the forward and reverse process one may obtain an efficient density estimator

\begin{remark}[Reverse diffusion vs.\ time-reversed diffusion]
It should be noted that the term ``reverse'' diffusion for $Y$ is somewhat misleading as it  
differs from the \emph{time-reversed} diffusion 
in the sense of Haussmann and Pardoux \cite{HP86}. For specifying the dynamics of the latter one explicitly needs the transition
density of $X$.  In contrast, the SDE dynamics of $Y$  is  straightforwardly inferred from the SDE dynamics of $X$ and  has usually similar regularity properties.  A key advantage of our ``reverse'' diffusion is that it can be constructed far more simply than the \emph{time-reversed} diffusion 
 in \cite{HP86}. As a consequence,
integrals of the form $\int g(x)\,q(0,x,T,\cdot)\,dx$ can be 
computed by simulation of ``reverse'' stochastic representations  
involving $Y$, more simply than by representations relying on the real 
``time-reversed'' diffusion.
Although the term ``reverse''  might thus be considered kind of a misnomer, it is  nonetheless  maintained  in this paper because it stems from our main
background references \cite{BS2014} and \cite{MSS2004}.    
\end{remark}
In the sequel, we make the following assumptions. 
\begin{assump} 
\label{ass:qrho}
Let the transition density $q$, the densities $\rho_0,$ $\rho_T$ and their respective  supports $\mathsf{S}_0,$ $\mathsf{S}_T$ be as in Theorem~\ref{pavon}. For technical reasons  we moreover assume that the compact sets
$\mathsf{S}_0$ and $\mathsf{S}_T$ are {\em connected}.
Let then for all \((x,z)\in \mathsf{S}_{0}\times \mathsf{S}_{T},\)
\[
0<q_{\min}\leq q(0,x;T,z) \leq q_{\max}<\infty,\quad
0<Q_{\min}\leq Q_0(x), Q_T(z)\leq Q_{\max}<\infty 
\]
with $Q_T(x):=\int_{\mathsf{S}_{T}} q(0,x;T,z)\,dz$ and $Q_0(z):=\int_{\mathsf{S}_{0}} q(0,x;T,z)\,dx.$
Moreover, we assume that
\[
0<\rho_{\min}\leq \rho_0(x),\rho_T(z)\leq \rho_{\max}<\infty.
\]
\end{assump}
It follows from Theorem~\ref{pavon} that in this case the solution of the SBP is unique (up to a scaling factor) and that the fix point  \(g^\star\) of (\ref{Cop}) (in the usual sense) is unique up to a multiplicative constant.
In order to enforce  complete uniqueness, we normalize  \(g^\star\). 
\begin{assump} 
\label{ass:nuT}
Assume that \(g^\star(z)=\rho_T(z)/\nu_T(z)\) integrates to \(1,\) that is,
\[
\int_{\mathsf{S}_{T}}g^\star(z)\,dz=1.
\]
\end{assump}
Under Assumption~\ref{ass:nuT}, we have \(1=\int_{\mathsf{S}_{T}} g^\star(z)\,dz=\int_{\mathsf{S}_{0}}Q_T(x)\nu_{0}(x)\,dx\) due to (\ref{SPT}), and consequently
\begin{eqnarray}
\label{eq:prior-g}
g^\star_{\min}\leq g^\star(z)\leq g^\star_{\max},\quad z\in \mathsf{S}_{T},
\end{eqnarray}
where \(g^\star_{\min}=q_{\min}/Q_{\max},\) \(g^\star_{\max}=q_{\max}/Q_{\min}.\) 
Under Assumption~\ref{ass:qrho},  we  also have the estimates%
\begin{align}
\label{eq:E-bounds}
0  &  <q_{\min}f_{\min}\leq\mathcal{E}_{T}(f),\mathcal{E}_{0}(f)\leq q_{\max} f_{\max}.
\end{align}

Let $K$ be a continuous nonnegative  kernel on $\mathbb{R}%
^{d}$ and let $\phi_0$ be a density on $\overline U_0$, for a bounded open set $U_0$ $\supset$ $\mathsf{S}_{0}$, which is bounded away from zero on $\mathsf{S}_{0}$.
 For obtaining an approximation to $\mathcal{E}_{T}(f)$  for any $f\in\mathcal{L}_{+}^{\infty}(\mathsf{S}_{T}),$
we use  a kernel-type regression estimate. First, we generate a sample  \(x^1,\ldots,x^N\sim \phi_0,\)  fix \(\delta=\delta_N\) and define 
\begin{eqnarray}
\mathcal{E}_{T}^{N}[f]:=
\begin{cases}
Q_T S_N[\rho_T f]/S_N[1_{\mathsf{S}_{T}}], & S_N[1_{\mathsf{S}_{T}}]>0,
\\
Q_{\min}\rho
_{\min}f_{\min}, & S_N[1_{\mathsf{S}_{T}}]=0
\end{cases}
\end{eqnarray}
where \(f_{\min} = \inf_{\mathsf{S}_{T}}f\) and
\[
S_N[g](x):=\frac{1}{N}\sum_{i=1}^{N}K((x-x^{i})/\delta
)g(X_{T}^{0,x^{i}}).
\]
 Note that \(S_N[1_{\mathsf{S}_{T}}]=0\) implies \(S_N[\rho_T f]=0\) and hence our definition of the estimate is natural.
Similarly, for any $f\in\mathcal{L}_{+}^{\infty}(\mathsf{S}_{0}),$
we sample $(Y_{T}^{z^{i}},\mathcal{Y}_{T}^{z^{i}})$, $i=1,\ldots,N,$ with
$z^{1},\ldots,z^N\sim\phi_T$, where \(\phi_T\) is a  density on $\overline U_T$  with  $U_T$ $\supset$ \(\mathsf{S}_{T}\)  being a bounded open set, and which is bounded away from zero on \(\mathsf{S}_{T}\). We then set%
\begin{eqnarray}
\mathcal{E}_{0}^{N}[f]:=
\begin{cases}
Q_0 \widetilde{S}_N[\rho_0 f]/\widetilde{S}_N[1_{\mathsf{S}_{0}}], & \widetilde{S}_N[1_{\mathsf{S}_{0}}]>0,
\\
Q_{\min}\rho
_{\min}f_{\min}, & \widetilde{S}_N[1_{\mathsf{S}_{0}}]=0.
\end{cases}
\end{eqnarray}
where  \(f_{\min} = \essinf_{\mathsf{S}_{0}}f,\)
\[
\widetilde{S}_N[g](z)=\frac{1}{N}\sum_{i=1}^{N}K((z-z^{i})/\delta
)g(Y_{T}^{z_{i}})\mathcal{Y}_{T}^{z_{i}}.
\]
Note that by construction, we have the lower bounds%
\begin{align}
Q_{\min}\rho_{\min}f_{\min } \leq \mathcal{E}_{T}^{N}[f]     \leq  \rho_{\max} f_{\max}, \quad
Q_{\min}
\rho_{\min}f_{\min } \leq \mathcal{E}_{0}^{N}[f]     \leq  \rho_{\max} f_{\max}{Q_{\max}}.  
\label{eq:EN-bounds}
\end{align}
The above kernel approximations 
result in an approximation of the operator
$\mathcal{C}$ in (\ref{cdec}) by,
\begin{equation}
\mathcal{C}^{N}:=\mathcal{E}_{0}^{N}\circ\mathcal{D}_{0}%
\circ\mathcal{E}_{T}^{N}\circ\mathcal{D}_{T}.\label{Cappr}%
\end{equation}
Note that, as well as $\mathcal{C}$, its approximation $\mathcal{C}^{N}$
is also positive homogeneous.
Finally,  consider for an arbitrarily fixed $g_{0}\in\mathcal{L}_{+}^{\infty}%
(\mathsf{S}_{T})$ 
%\textcolor{red}{with $g^{\star}_{\min}\leq g_0\leq g^{\star}_{\max}$ %and $\|g_{0}\|_{L_1}=1$,} 
the sequence of approximations

\begin{equation}
\label{eq:picar-iter}
\widehat{g}%
_{\ell}:=\mathcal{T}_{[g^{\star}_{\min},g^{\star}_{\max}]}[\widetilde g_\ell], \quad \widetilde g_\ell=\mathcal{C}^{N}[\widehat{g}_{\ell-1}]/\|\mathcal{C}^{N}[\widehat{g}_{\ell-1}]\|_{L_1}, \quad \ell\geq1
\end{equation}
with  $\widehat{g}_{0}:=g_{0}.$
Here, for any \(0<a<b<\infty, \)  \(\mathcal{T}_{[a,b]}\) is a truncation operator of the form
\[
\mathcal{T}_{[a,b]}[f]:=
\begin{cases}
a, & f(x)\leq a,
\\
f(x), & a < f(x) \leq  b,
\\
b,   &  f(x)> b.
\end{cases}
\]
Finally, we define the corresponding approximating sequence 
for \(\nu_T\) as \(\widetilde\nu_{T}=\rho_{T}/\widehat{g}%
_{\ell}\) for some $\ell>1.$

\section{Convergence analysis}\label{sec:conv}
\subsection{Upper bounds}
Following \cite{CGP2016}, $\mathcal{E}_{0}$ and  $\mathcal{E}_{T}$ are $d_H$-contractions
with contraction coefficients \(\kappa(  \mathcal{E}_{0})\) and \(\kappa(  \mathcal{E}_{T}),\) respectively satisfying
\[
\max\{\kappa(  \mathcal{E}_{0}), \kappa(  \mathcal{E}_{T})\} \leq \tanh\left(\frac{1}{2}\log(  q_{\max}/q_{\min})\right)  <1.
\]
 Moreover,
$\mathcal{D}_{0},$ \(\mathcal{D}_{T}\) are $d_H$-isometries on $\mathcal{L}_{+}^{\infty}(\mathsf{S}_{0})$ and $\mathcal{L}_{+}^{\infty}(\mathsf{S}_{T}),$ respectively. Hence, $\mathcal{C}$ in (\ref{cdec}) is a contraction
on $\mathcal{L}_{+}^{\infty}(\mathsf{S}_{T})$ with contraction coefficient $\kappa\left(
\mathcal{C}\right)  \leq$ $\tanh^{2}(\frac{1}{2}\log( q_{\max}/q_{\min} 
)  )<1$ with respect to the Hilbert metric $d_H$. The following proposition holds.

\begin{proposition}
\label{prop:conv-main} 
Let the kernel function $K:\mathbb{R}^{d}\to
\mathbb{R}_{+}$ satisfy

\begin{itemize}
\item $\|K\|_{\infty}=K_{\infty}<\infty,$ $\int K(x)\,dx=1$ and $\int
x_{i}K(x)\,dx=0$ for $i=1,\ldots,d;$

\item $K$ has a support contained in $\left[ -\frac{1}{2},\frac{1}{2}\right]
^{d}$;

\item For any fixed $\gamma>0,$ the class $\mathcal{K}=\{x\mapsto
K(\gamma(x-z)):z\in\mathbb{R}^{d}\}$ is a measurable VC-type class of
functions from $\mathbb{R}^{d}$ to $\mathbb{R}$.
\end{itemize}

Suppose that $\min\left(  \inf_{S_{0}}\phi_{0},\inf_{S_{T}}\phi_{T}\right)
\geq\phi_{\min}>0$ and that
\begin{align*}
q(0,\cdot;T,z)\phi_{0}(\cdot)  & \in\mathcal{H}^{1,\alpha}(\overline{U}%
_{0})\text{ \ for any }z\in\mathsf{S}_{T},\text{\ }\\
q(0,x;T,\cdot)\phi_{T}(\cdot)  & \in\mathcal{H}^{1,\alpha}(\overline{U}%
_{T})\text{ \ for any }x\in\mathsf{S}_{0},
\end{align*}
such that moreover%
\[
\max\left(  \sup_{z\in\mathsf{S}_{T}}\left\Vert q(0,\cdot;T,z)\phi_{0}%
(\cdot)\right\Vert _{\mathcal{H}^{1,\alpha}(\overline{U}_{0})},\sup
_{x\in\mathsf{S}_{0}}\left\Vert q(0,x;T,\cdot)\phi_{T}(\cdot)\right\Vert
_{\mathcal{H}^{1,\alpha}(\overline{U}_{T})}\right)  \leq B_{q}%
\]
%\db{What is $\overline{U}_{0}$ and $\overline{U}_{T}$ here ? Where were they defined ?}
for some $\alpha\in(0,1].$ For a recap on H\"older spaces we refer to Appendix~\ref{SmoothC}. Then we have under the choice $\delta
_{N}=N^{-2/(2(1+\alpha)+d)}$,
\[
\mathbb{E}\left[  d_{H}(\widehat{g}_{k},g^{\star})\right]  \lesssim
(1-\kappa(\mathcal{C}))^{-1}N^{-\frac{1+\alpha}{2(1+\alpha)+d}}+(\kappa
(\mathcal{C}))^{k}d_{H}(g_{0},g^{\star})
\]
where $\lesssim$ stands for inequality up to a constant depending on $q_{\min
},q_{\max},\rho_{\min},\rho_{\max}$ and $B_{q}.$

\end{proposition}

\begin{cor} \label{cor:rate} Take 
\(k\geq \frac{1+\alpha}{2(1+\alpha)+d}\log (N)/\log(1/\kappa(\mathcal{C}))\) then we have 
\begin{align*}
\mathbb{E}\left[  d_{H}(\widehat{g}_{k},g^{\star})\right]   &    \lesssim N^{-\frac{1+\alpha}{2(1+\alpha)+d}}.
\end{align*}
Moreover, it holds 
\begin{eqnarray}
\label{eq:rate}
\mathbb{E}\left[\|\widehat{g}_{k}-g^{\star}\|_\infty\right]\lesssim N^{-\frac{1+\alpha}{2(1+\alpha)+d}}.
\end{eqnarray}
Here, \(\lesssim\) stands for inequality up to a constant depending on \(q_{\min},q_{\max},\)  \(\rho_{\min},\rho_{\max}\) and \(B_q.\)
\end{cor}
\subsection{Lower bounds}
We  present now lower bounds showing that the rates of Corollary~\ref{cor:rate} cannot be improved in general. 
For this it is enough to work under the hypothetical assumption
that $\mathcal{E}_{0}\circ\mathcal{D}_{0}$ is known exactly. That is, rather
than (\ref{Cappr}) we consider the iterative procedure described in
Section~\ref{sec:aip} with respect to the noisy operator%
\[
\mathcal{E}_{0}\circ\mathcal{D}_{0}\circ\mathcal{E}_{T}^{N}\circ
\mathcal{D}_{T}.
\]
\begin{thm}
\label{thm:lb}
Fix some $\alpha\in (0,1]$ and define a class $\mathcal{Q}_{\alpha}$ $\equiv$ \(\mathcal{Q}_{\alpha}(q_{\min},q_{\max})\) of continuous and strictly positive transition densities $q$ on \(\mathbb{R}^d\times \mathbb{R}^d\)  that satisfy
\[
0<q_{\min}\leq q(0,x;T,z) \leq q_{\max},\quad x,z\in [0,1]^d\times [0,1]^d,
\] 
considered that $\mathsf{S}_{0}$ $=$ $\mathsf{S}_{T}$ $=$ $[0,1]^d$. 
Suppose that both $\rho_{0}$ and $\rho_{T}$ are distribution densities on
$[0,1]^{d}$ satisfying Assumption ~\ref{ass:qrho}. 
Suppose that
\[
 \sup_{z\in [0,1]^d}\left\Vert q(0,\cdot;T,z)\rho_{0}%
(\cdot)\right\Vert _{\mathcal{H}^{1,\alpha}([0,1]^d)}  \leq B_{\alpha},\qquad q\in\mathcal{Q}_{\alpha}.
\]
It then holds
\begin{eqnarray}
\label{eq:lower-bound}
\inf_{\widehat{g}}\sup_{q\in \mathcal{Q}_{\alpha}}\mathbb{E}_q\left[d_H(\widehat{g},g^{\star})\right]\gtrsim N^{-\frac{1+\alpha}{2(1+\alpha)+d}}
\end{eqnarray}
where \(\mathbb{E}_q\) stands for expectation under the joint distribution of \((X_0,X_T)\sim \rho_0(x)q(0,x;T,z)\) and infimum is taken over all  estimates $\widehat{g}$ of \(g^{\star}\) solving 
\begin{eqnarray*}
\int_{\mathsf{S}_{0}}%
\frac{\rho_{0}(x)}{\int_{\mathsf{S}_{T}}q(0,x;T,z')\frac{\rho_{T}(z')}{g^{\star}(z')}\, dz'}
q(0,x;T,z)\, dx=g^{\star}(z)
\end{eqnarray*}
based on a iid sample from \((X_0,X_T)\) of the length \(N.\)
\end{thm}

\section{Simulation of Schr\"odinger Bridges}\label{SBsim}

It is known (see e.g. \cite{daipra1991}) that the Schr\"odinger Markov process $X$  can be constructed as a solution of the following SDE: 
\begin{equation}\label{h_dyn}
    d X_{t}
    = \left( b(X_{t}, t) + \sigma(X_{t}, t) \sigma(X_{t}, t)^\top \nabla\log h(X_{t}, t) \right) \, d t + \sigma(X_{t}, t) \, d W_{t}
    \end{equation}
with \(X_{0} \sim \rho_0\), 
where
\[
    h(w, t)
    = \int_{\mathsf{S}_T} q(t,w;T,y) \, \nu_{T}  (y)\,dy
\]
and $q$ is the transition density of reference process corresponding to
(\ref{h_dyn}) with $h$ being constant.
Let \(\widetilde{\nu}_T\) be an  estimate for \(\nu_T\) obtained 
by the procedure in Section~\ref{sec:aip}. We then have
\[
        \nu_{\min} \leq \nu_T(y), \, \widetilde{\nu}_T(y) \leq \nu_{\max} \quad \text{for all } y \in \mathsf{S}_T
    \]
   for some \(\nu_{\min},\nu_{\max}>0.\) Consider the approximated process
\[
    d \widetilde{X}_{t}
    = \left( b(\widetilde{X}_{t}, t) + \sigma(\widetilde{X}_{t}, t) \sigma(\widetilde{X}_{t}, t)^\top \nabla\log \widetilde{h}(\widetilde{X}_{t}, t) \right) \, d t + \sigma(\widetilde{X}_{t}, t) \, d W_{t}
    \]
with \(X_{0} \sim \rho_0\) 
where
\[
    \widetilde{h}(w, t)
    = \int_{\mathsf{S}_T} q(t,w;T,y) \, \widetilde{\nu}_{T}  (y)\,dy.
\] 
Let $\Delta(x,t) := \nabla_x \log h(x,t) - \nabla_x \log \widetilde{h}(x,t)$, and assume $\sigma(x,t)$ is invertible. The KL divergence between the laws of two diffusion processes $\mathbb{P}_{[0,T-\delta]}$ and $\widetilde{\mathbb{P}}_{[0,T-\delta]}$ (on the time interval $[0,T-\delta]$) can be expressed using Girsanov's theorem in terms of $\Delta:$
$$
\mathrm{KL}(\mathbb{P}_{[0,T-\delta]} \,\|\, \widetilde{\mathbb{P}}_{[0,T-\delta]})
= \frac{1}{2} \mathbb{E}^{\mathbb{P}} \left[ \int_0^{T-\delta} \left\| \sigma^{-1}(X_t,t) \sigma(X_t,t)\sigma^\top(X_t,t) \Delta(X_t,t) \right\|^2 dt \right].
$$
Since $\sigma^{-1} \sigma \sigma^\top = \sigma^\top$,  we simplify
$$
\mathrm{KL}(\mathbb{P}_{[0,T-\delta]} \,\|\, \widetilde{\mathbb{P}}_{[0,T-\delta]})
= \frac{1}{2} \mathbb{E}^{\mathbb{P}} \left[ \int_0^{T-\delta} \left\| \sigma^\top(X_t,t) \Delta(X_t,t) \right\|^2 dt \right].
$$
Furthermore, we have
\begin{multline*}
\mathbb{E} \left[ \left\| \nabla_x \log \widetilde{h}(X_t,t) - \nabla_x \log h(X_t,t) \right\|^2 \right]
\leq \frac{\| \widetilde{\nu}_T - \nu_T \|_{\mathsf{S}_T}^2}{\nu_{\min}^2} 
\\
\cdot \mathbb{E} \left[
\left( \sup_{y \in \mathsf{S}_T} \| \nabla_x \log q(t,X_t;T,y) \| + \| \nabla_x \log h(X_t,t) \| \right)^2
\right].
\end{multline*}
Now assume that the potential \( \nu_T \) is supported on \( \mathsf{S}_T \subseteq B_R(0) \subset \mathbb{R}^d \)
and the transition density \(q\) of the reference process satisfies
\[
    \sup_{y \in B_R(0)} \| \nabla_x \log q(t,x;T,y) \| \leq C\cdot\frac{\|x\| + R}{T - t}, \quad x\in \mathbb{R}^d
\]    
for some absolute constant \(C>0.\)  Under the above assumptions,
\[
\left\| \nabla_x \log h(x,t) \right\| \leq C \cdot \frac{\|x\| + R}{T - t} 
\]
and 
\[
\mathbb{E} \left[ \left\| \nabla_x \log \widetilde{h}(X_t,t) - \nabla_x \log h(X_t,t) \right\|^2 \right]
\leq 2\cdot C\cdot \frac{\| \widetilde{\nu}_T - \nu_T \|_{\mathsf{S}_T}^2}{\nu_{\min}^2 (T - t)^2} \cdot \mathbb{E} \left[ (\|X_t\| + R)^2 \right].
\]
Assume that $\|\sigma\|_{\infty}\leq \sigma_{\max}<\infty,$ then
\[
\mathrm{KL}(\mathbb{P}_{[0,T-\delta]} \,\|\, \widetilde{\mathbb{P}}_{[0,T-\delta]})\lesssim \sigma_{\max}^2\delta^{-1}\| \widetilde{\nu}_T - \nu_T \|_{\mathsf{S}_T}^2. 
\]
So we see that the bound explodes if $\delta\to 0$ meaning that simulation of the SB can be difficult especially under estimated $\nu_T.$ If one only needs some expected functionals of the SB depending on its finite dimensional distributions, we propose a more efficient way  of estimation in the next section.
\section{Forward-Reverse simulation for reciprocal and Schr\"{o}dinger processes}
\label{FRSP} 

In Appendix~\ref{sec:rp} we have recapitulated the concept of reciprocal processes
in general and  Schr\"{o}dinger processes in particular, being reciprocal Markov process with endpoint distribution 
satisfying  (\ref{eq:schro_factor}).
In this section we propose simulation based approaches for
estimating functionals of the form (\ref{eq:schroe-findim}), hence the finite dimensional distributions of such processes. The here proposed methods may be seen as an application of the forward-reverse approach developed 
in \cite{BS2014}, recapitulated in Appendices~\ref{subsec:reverse_diffusion_construction}-\ref{RecapFR}, combined  with the simulation based construction of the   Schr\"{o}dinger measures or potentials developed in Section~\ref{sec:aip}.

\subsection{Stochastic representations for reciprocal processes}

By combining (\ref{eq:schroe-findim}) (see also (\ref{endpdis})) with (\ref{FRrep}) we
immediately obtain an FR stochastic representation for the finite dimensional
distributions of a reciprocal process due to a begin-endpoint measure
$\mu(dx,dz)$:%
\begin{multline}
\mathbb{E}\left[  g(X_{0}^{\mu},X_{t_{1}}^{\mu},\ldots,X_{t_{n}}^{\mu}%
,X_{T}^{\mu})\right]  =\int_{\mathbb{R}^{d}\times\mathbb{R}^{d}}%
\mu(dx,dz)\mathcal{E}(g(x,\cdot,z);x,z)\label{FRschroe}\\
=\int_{\mathbb{R}^{d}\times\mathbb{R}^{d}}\mu(dx,dz)\frac{H(g(x,\cdot
,z);x,z)}{q(0,x,T,z)}\\
=\int_{\mathbb{R}^{d}\times\mathbb{R}^{d}}\mu(dx,dz)\frac{\lim_{\varepsilon
\downarrow0}H_{\varepsilon}(g(x,\cdot,z);x,z)}{q(0,x,T,z)}%
\end{multline}
for any bounded measurable $g:\left(  \mathbb{R}^{d}\right)  ^{(K+L+1)}%
\rightarrow\mathbb{R}$.

Due to the FR simulation procedure for the representation (\ref{FRrep}),
%by \cite{BS2014},
a straightforward simulation procedure for (\ref{FRschroe}) suggests itself:
One may sample a number of pairs $(X_{0}^{(r)},Z_{T}^{(r)}),$ $r=1,\ldots,R,$
from the distribution $\mu.$ Then for each particular drawing $r$ one may
approximate $H(g(X_{0}^{(r)},\cdot,Z_{T}^{(r)});X_{0}^{(r)},Z_{T}^{(r)})$ and
$q(0,X_{0}^{(r)},T,Z_{T}^{(r)})$, and hence $\mathcal{E}(g(X_{0}^{(r)}%
,\cdot,Z_{T}^{(r)});X_{0}^{(r)},Z_{T}^{(r)})$ in (\ref{FRrep}) for the pair
$(X_{0}^{(r)},Z_{T}^{(r)}),$ using $N$ trajectories of $X$ and $N$
trajectories of $(Y,\mathcal{Y})$ according to (\ref{SDE}) and (\ref{Yrev}).
One finally takes the average over $R$ estimations in order to obtain an
estimate of (\ref{FRschroe}). Obviously, this nested simulation procedure will
be generally slow as it requires the simulation of order $NR$ trajectories. In
the next section we propose a more efficient (non-nested) Monte Carlo procedure
for computing (\ref{FRschroe}) in the case of a Schr\"odinger process.

\subsection{Stochastic representations for Schr\"{o}dinger processes}

A Schr\"{o}dinger Markov process is determined by a begin-endpoint
distribution $\mu$ of the form (\ref{eq:schro_factor}) due to $\sigma$-finite Borel
measures $\nu_{0}$ and $\nu_{T}$ satisfying%
\[
\int_{\mathbb{R}^{d}\times\mathbb{R}^{d}}\nu_{0}(dx)q(0,x,T,z)\nu_{T}(dz)=1.
\]
Conversely, any pair of Borel measures $\widetilde{\nu}_{0}$ and
$\widetilde{\nu}_{T}$ on $\mathbb{R}^{d}$ with%
\begin{equation}
0<c_{0,T}^{-1}:=\int_{\mathbb{R}^{d}\times\mathbb{R}^{d}}\widetilde{\nu}%
_{0}(dx)q(0,x,T,z)\widetilde{\nu}_{T}(dz)<\infty\label{c0t}%
\end{equation}
gives rise to a begin-endpoint distribution of the form (\ref{eq:schro_factor}) due to%
\begin{align}
\mu(dx,dz)  &  =c_{0,T}\widetilde{\nu}_{0}(dx)q(0,x,T,z)\widetilde{\nu}%
_{T}(dz)\label{schroe11}\\
&  =:\nu_{0}(dx)q(0,x,T,z)\nu_{T}(dz)\nonumber
\end{align}
with $\nu_{0}=c_{0,T}^{1/2}\widetilde{\nu}_{0}$ and $\nu_{T}=c_{0,T}%
^{1/2}\widetilde{\nu}_{T}.$ Obviously, if $\mu$ is defined via (\ref{c0t}) and
(\ref{schroe11}) for given $\widetilde{\nu}_{0}$ or $\widetilde{\nu}_{T},$ it
is invariant under scaling of $\widetilde{\nu}_{0}$ or $\widetilde{\nu}_{T}$
by an arbitrary positive constant. Thus if, moreover, $\widetilde{\nu}_{0}$ or
$\widetilde{\nu}_{T}$ is a finite measure, we may w.l.o.g. assume that it is a
probability measures.

We now assume that, either, we are given a pair of probability measures $\nu_0$ and $\nu_T$ that define an endpoint distribution $\mu$ in (\ref{eq:schro_factor}), or we are given $\mu$ and assume
that $\nu_0$ and $\nu_T$ are obtained via the approximation procedure of Section~\ref{sec:aip}. 

Let us abbreviate for $x,z\in\mathbb{R}^{d},$ and bounded $g:\left(
\mathbb{R}^{d}\right)  ^{(K+L+1)}\rightarrow\mathbb{R}$, the random variable%
\begin{multline*}
\zeta_{\varepsilon}(g\left(  x,\cdot,z\right)  ;X^{x},Y^{z},\mathcal{Y}%
^{z};x,z):=\\
g\left(  x,X_{s_{1}}^{x},\ldots,X_{s_{K-1}}^{x},X_{t^{\ast}}^{x}%
,Y_{\widehat{t}_{L-1}}^{z},\ldots,Y_{\widehat{t}_{1}}^{z},z\right)
K_{\varepsilon}(Y_{T-t^{\ast}}^{z}-X_{t^{\ast}}^{x})\mathcal{Y}_{T-t^{\ast}%
}^{z},
\end{multline*}
hence
\begin{align*}
\lim_{\varepsilon\downarrow0}\mathbb{E}\left[  \zeta_{\varepsilon}(g\left(
x,\cdot,z\right)  ;X^{x},Y^{z},\mathcal{Y}^{z};x,z)\right]   &  =\lim
_{\varepsilon\downarrow0}H_{\varepsilon}(g\left(  x,\cdot,z\right)  ;x,z)\\
&  =H(g\left(  x,\cdot,z\right)  ;x,z).
\end{align*}
For $\mu$ defined by (\ref{schroe11}) for given $\widetilde{\nu}_{0}$ and
$\widetilde{\nu}_{T},$ we may write by (\ref{FRrep}),
\begin{align}
\mathcal{E}^{\mu}\left(  g\right)   &  :=\mathbb{E}\left[  g(X_{0}^{\mu
},X_{t_{1}}^{\mu},\ldots,X_{t_{n}}^{\mu},X_{T}^{\mu})\right] \label{cale}\\
&  =\int_{\mathbb{R}^{d}\times\mathbb{R}^{d}}\mu(dx,dz)\mathcal{E}%
(g(x,\cdot,z);x,z)\nonumber\\
&  =c_{0,T}\lim_{\varepsilon\downarrow0}\int_{\mathbb{R}^{d}\times
\mathbb{R}^{d}}\widetilde{\nu}_{0}(dx)\widetilde{\nu}_{T}(dz)\mathbb{E}\left[
\zeta_{\varepsilon}(g\left(  x,\cdot,z\right)  ;X^{x},Y^{z},\mathcal{Y}%
^{z};x,z)\right]  .\nonumber
\end{align}
For example, if $\widetilde{\nu}_{T}$ has a density, i.e. $\widetilde{\nu}%
_{T}(dz)=\widetilde{\nu}_{T}(z)dz,$ and $\widetilde{\nu}_{0}$ is a probability
measure, then the constant%
\[
c_{0,T}^{-1}=\mathbb{E}\left[  \widetilde{\nu}_{T}(X_{T}^{0,U})\right]
\]
with $U\sim$ $\widetilde{\nu}_{0}$ may usually be computed accurately by
standard Monte Carlo. Furthermore, if both $\widetilde{\nu}_{0}$ and
$\widetilde{\nu}_{T}$ are probability measures, (\ref{cale}) has the
representation%
\begin{equation}
\mathcal{E}^{\mu}\left(  g\right)  =c_{0,T}\lim_{\varepsilon\downarrow
0}\mathbb{E}\left[  \zeta_{\varepsilon}(g\left(  U,\cdot,Z\right)
;X^{U},Y^{Z},\mathcal{Y}^{Z};U,Z)\right]  \label{prepe}%
\end{equation}
with $U\sim\widetilde{\nu}_{0}$ and $Z\sim\widetilde{\nu}_{T}.$ The
representation (\ref{prepe}) allows for the following simulation procedure:
Suppose that the points $U^{(r)}$ and $Z^{(r)},$ $r=1,\ldots,K,$ are simulated
i.i.d. from the probability measures $\widetilde{\nu}_{0},$ and
$\widetilde{\nu}_{T},$ respectively. Then, for each particular $r$ we sample a
Wiener processes $W^{(r)}$ and $\widetilde{W}^{(r)},$ and construct a forward
and reverse trajectory%
\[
X^{U^{(r)}}\text{ \ \ and \ \ }\left(  Y^{Z^{(r)}},\mathcal{Y}^{Z^{(r)}%
}\right)  ,\text{ \ \ }r=1,\ldots,R,
\]
respectively. We then consider the estimate%
\begin{align}
\widehat{\mathcal{E}}_{\varepsilon,R}^{\mu}(g)  &  :=\frac{c_{0,T}}{R^{2}}%
\sum_{r=1}^{R}\sum_{r^{\prime}=1}^{R}\zeta_{\varepsilon}(g\left(
U^{(r^{\prime})},\cdot,Z^{(r)}\right)  ;X^{U^{(r^{\prime})}},Y^{Z^{(r)}%
},\mathcal{Y}^{Z^{(r)}};U^{(r^{\prime})},Z^{(r)})\notag\\
&  =\frac{c_{0,T}}{R^{2}}\sum_{r=1}^{R}\sum_{r^{\prime}=1}^{R}g\left(
U^{(r^{\prime})},X_{s_{1}}^{U^{(r^{\prime})}},\ldots,X_{s_{K-1}}%
^{U^{(r^{\prime})}},X_{t^{\ast}}^{U^{(r^{\prime})}},Y_{\widehat{t}_{L-1}%
}^{Z^{(r)}},\ldots,Y_{\widehat{t}_{1}}^{Z^{(r)}},Z^{(r)}\right) \notag\\
&  \cdot K_{\varepsilon}(Y_{T-t^{\ast}}^{Z^{(r)}}-X_{t^{\ast}}^{U^{(r^{\prime
})}})\mathcal{Y}_{T-t^{\ast}}^{Z^{(r)}} \label{non_nest_est}%
\end{align}
which is a non-nested Monte Carlo estimator in fact.

\begin{remark}
Note that one has that%
\[
\mathbb{E}\left[  \zeta_{\varepsilon}(g\left(  x,\cdot,z\right)  ;X^{x}%
,Y^{z},\mathcal{Y}^{z};x,z)\right]  \lesssim q(0,x,T,z),
\]
where the right-hand-side is integrable with respect to $\widetilde{\nu}%
_{0}\otimes\widetilde{\nu}_{T}$ due to assumption (\ref{c0t}). In the case
where $\widetilde{\nu}_{0}$ or $\widetilde{\nu}_{T}$ is not a finite measure
one then may design a similar FR simulation procedure based on some importance
sampling or MCMC\ technique. We omit the details.
\end{remark}

\begin{remark}
The estimator  (\ref{non_nest_est}) due to a generic 
test functional $g$ allows for estimating the probability that  the Schr\"odinger Bridge process $X^\mu$ (see (\ref{h_dyn})) visits at arbitrarily chosen discrete times  arbitrarily chosen (Borel) regions. We underline that 
this estimator acts on trajectories generated by the reference process $X$ and its corresponding reverse process $Y$ only,
and thus simulation of the actual trajectories of the SB process $X^\mu$ is not needed for this purpose. 
Furthermore, in the previous section it is shown that
simulation of the real trajectories 
of $X^\mu$ via (\ref{h_dyn}) may be a delicate issue, particularly in cases where 
$\nu_T$ is approximated and time $t$ approaches the terminal time $T$. Moreover,
in (\ref{h_dyn})  one needs to compute $h$ at any time $0\leq t\leq T,$ 
which either requires knowledge of the transition density $q$ or 
requires extra sub-simulations at each simulated trajectory.
Further one could say that simulation of (\ref{h_dyn}) is related to simulation of conditional diffusion trajectories (e.g. see\cite{Schauer17}), which is known to be a delicate issue 
for similar reasons. 
\end{remark}

\section{Proofs}\label{proofs}
\subsection{Proof of Proposition~\ref{prop:conv-main}}
By the
contractivity of $\mathcal{C}$ and the continuity of the $L_1$-normalized $g^{\star}$ and $\widetilde{g}_{\ell}$ for $\ell\geq1,$ one has due to Lemma~\ref{Dhineq} and Corollary~\ref{HmTr},  %
\begin{align*}
d_{H}(\widehat{g}_{\ell},g^{\star})  &  \leq  d_{H}(\widetilde{g}_{\ell},g^{\star})
\\
&=d_{H}(\mathcal{C}^{N}%
(\widehat{g}_{\ell-1}),\mathcal{C}(g^{\star})  )\\
&  \leq d_{H}(\mathcal{C}^{N}(\widehat{g}_{\ell-1}),\mathcal{C}(
\widehat{g}_{\ell-1})  )+\kappa( \mathcal{C})  d_{H}%
(\widehat{g}_{\ell-1},g^{\star}).
\end{align*}
Hence for any $k\geq 1,$
\begin{align}
d_{H}(\widehat{g}_{k},g^{\star})  &  \leq\sum_{i=1}^{k}\kappa(
\mathcal{C})  ^{i-1}\widehat{\varepsilon}_{k-i}+\kappa(
\mathcal{C})  ^{k}d_{H}(g_{0},g^{\star})\nonumber
\end{align}
where \(\widehat\varepsilon_{\ell}:= d_{H}(\mathcal{C}^{N}(\widehat{g}%
_{\ell}),\mathcal{C}( \widehat{g}_{\ell})  ),\) \(\ell\geq 0.\) 
For a generic $f\in\mathcal{L}_{+}^{\infty}(S_{T})$ one has, by (\ref{cdec}),
(\ref{Cappr}), the triangle inequality for $d_{H},$ the fact that
$\mathcal{D}_{0}$ is an $d_{H}$-isometry on $\mathcal{L}_{+}^{\infty}(S_{0}),$
and the contractivity of $\mathcal{E}_{0},$
\begin{align*}
  d_{H}(\mathcal{C}^{N}(f),\mathcal{C}(f) )&=d_{H}(
\mathcal{E}_{0}^{N}(  \mathcal{E}_{T}^{N}(f^{-1}%
)^{-1})  ,\mathcal{E}_{0}(  \mathcal{E}_{T}(f^{-1})^{-1})) \nonumber\\
&  \leq d_{H}(  \mathcal{E}_{0}^{N}(  \mathcal{E}%
_{T}^{N}(f^{-1})^{-1})  ,\mathcal{E}_{0}(  \mathcal{E}%
_{T}^{N}(f^{-1})^{-1}))  +\kappa (  \mathcal{E}_{0})  d_{H}(  \mathcal{E}%
_{T}^{N}(f^{-1}),\mathcal{E}_{T}(f^{-1})) \nonumber\\
&  \equiv\text{Term}_{1}+\text{Term}_{2}.
\end{align*}
For any $g\in\mathcal{L}_{+}^{\infty}(S_{0})$ it holds due to \eqref{eq:E-bounds} and \eqref{eq:EN-bounds} and Lemma~\ref{lemH},
\begin{align}
&  d_{H}\left(  \mathcal{E}_{0}^{N}(g),\mathcal{E}_{0}(g)\right)
\nonumber  \leq\frac{2}{g_{\min } \min(q_{\min},\rho_{\min}Q_{\min})}\Vert \mathcal{E}%
_{0}^{N}(g)-\mathcal{E}_{0}(g)\Vert _{\infty}. \label{dHe2}%
\end{align}
and with  $g=1\big/\mathcal{E}_{T}^{N}(f^{-1})$ we get \(g_{\min}=f_{\min }/\rho_{\max}.\)
Hence 
\begin{align}
\text{Term}_{1}   \leq \frac{2\rho_{\max} f_{\max}}{ Q_{\min} \rho_{\min} f_{\min } \min(q_{\min},\rho_{\min}Q_{\min})}\left\Vert \mathcal{E}_{0}^{N}\left(\frac
{\inf\mathcal{E}_{T}^{N}(f^{-1})}{\mathcal{E}_{T}^{N}(f^{-1}%
)}\right)-\mathcal{E}_{0}\left(\frac{\inf\mathcal{E}_{T}^{N}(f^{-1})}%
{\mathcal{E}_{T}^{N}(f^{-1})}\right)\right\Vert _{\infty}.\nonumber
\end{align}
Similarly, we have  for any $f\in\mathcal{L}_{+}^{\infty}(S_{T}),$%
\begin{align}
\text{Term}_{2}  \leq \kappa (  \mathcal{E}_{0}) \frac{2f_{\max}}{ f_{\min} \min(q_{\min},\rho_{\min}Q_{\min})}\left\Vert \mathcal{E}_{T}%
^{N}\left(\frac{\inf f}{f}\right)-\mathcal{E}_{T}\left(\frac{\inf f}{f}\right)\right\Vert
_{\infty}.\nonumber
\end{align}
Now using the fact that  by construction (see \eqref{eq:picar-iter}) and \eqref{eq:prior-g},
\begin{eqnarray*}
\frac{\widehat g_{\ell,\max}}{\widehat g_{\ell,\min}}\leq \frac{g^\star_{\max}}{ g^\star_{\min}}
\end{eqnarray*}
we derive
\begin{multline*}
d_{H}(\mathcal{C}^{N}(\widehat{g}%
_{\ell}),\mathcal{C}( \widehat{g}_{\ell})  )\leq A_0\left\Vert \mathcal{E}_{0}^{N}\left(\frac
{\inf\mathcal{E}_{T}^{N}( \widehat{g}_{\ell}^{-1})}{\mathcal{E}_{T}^{N}( \widehat{g}_{\ell}^{-1}%
)}\right)-\mathcal{E}_{0}\left(\frac{\inf\mathcal{E}_{T}^{N}( \widehat{g}_{\ell}^{-1})}%
{\mathcal{E}_{T}^{N}( \widehat{g}_{\ell}^{-1})}\right)\right\Vert _{\infty}
\\
+A_T\left\Vert \mathcal{E}_{T}%
^{N}\left(\frac{\inf  \widehat{g}_{\ell}}{ \widehat{g}_{\ell}}\right)-\mathcal{E}_{T}\left(\frac{\inf  \widehat{g}_{\ell}}{ \widehat{g}_{\ell}}\right)\right\Vert
_{\infty}
\end{multline*}
with
\begin{eqnarray*}
A_0=\frac{2\rho_{\max}}{Q_{\min} \rho_{\min}  \min(q_{\min},\rho_{\min}Q_{\min})}\frac{g^\star_{\max}}{ g^\star_{\min}}
\end{eqnarray*}
and
\begin{eqnarray*}
A_T=\kappa (  \mathcal{E}_{0}) \frac{2}{ \min(q_{\min},\rho_{\min}Q_{\min})}\frac{g^\star_{\max}}{ g^\star_{\min}}.
\end{eqnarray*}
Denote now \(\mathcal{F}_{\ell}\) the \(\sigma\) algebra generated by  the estimates \(\widehat g_1,\ldots,  \widehat g_\ell\) with \(\mathcal{F}_{0}=(\Omega,\varnothing)\) by definition. Then it holds 
\begin{eqnarray*}
\mathbb{E}[d_{H}(\widehat{g}_{k},g^{\star})]\leq  \mathbb{E}\Bigl[\sum_{i=1}^{k}\kappa(
\mathcal{C})  ^{i-1}\mathbb{E}[\widehat{\varepsilon}_{k-i}|\mathcal{F}_{k-i}]\Bigr]+\kappa(
\mathcal{C})  ^{k}d_{H}(g_{0},g^{\star})
\end{eqnarray*}
with 
\begin{equation*}
|\mathbb{E}[\widehat{\varepsilon}_{\ell}|\mathcal{F}_{\ell}]|\leq  A_0\mathbb{E}[\left\Vert \mathcal{E}_{0}^{N}(g_{0,\ell})-\mathcal{E}_{0}(g_{0,\ell})\right\Vert _{\infty}|\mathcal{F}_{\ell}]
+A_T\mathbb{E}[\left\Vert \mathcal{E}_{T}%
^{N}(g_{1,\ell})-\mathcal{E}_{T}(g_{1,\ell})\right\Vert
_{\infty}|\mathcal{F}_{\ell}]
\end{equation*}
and 
\[
g_{0,\ell}=\frac
{\inf\mathcal{E}_{T}^{N}( \widehat{g}_{\ell}^{-1})}{\mathcal{E}_{T}^{N}( \widehat{g}_{\ell}^{-1} 
)}\leq 1, \quad g_{1,\ell}=\frac{\inf  \widehat{g}_{\ell}}{ \widehat{g}_{\ell}}\leq 1.
\]
Furthermore, note that

\begin{multline*}
\left\Vert \mathcal{E}_{0}(\cdot)\phi_{0}(\cdot)\right\Vert _{\mathcal{H}%
^{1,\alpha}(\overline{U}_{0})}=\left\Vert \int_{\mathsf{S}_{T}}q(0,\cdot
,T,z)\phi_{0}(\cdot)\,\nu_{T}(z)\,dz\right\Vert _{\mathcal{H}^{1,\alpha
}(\overline{U}_{0})}\\
\leq\int_{\mathsf{S}_{T}}\left\Vert q(0,\cdot,T,z)\phi_{0}(\cdot)\right\Vert
_{\mathcal{H}^{1,\alpha}(\overline{U}_{0})}\,\nu_{T}(z)\,dz\leq B_{q}%
\rho_{\max}/g_{\min}^{\star}%
\end{multline*}
and
\begin{multline*}
\left\Vert \mathcal{E}_{T}(\cdot)\phi_{T}(\cdot)\right\Vert _{\mathcal{H}%
^{1,\alpha}(\overline{U}_{T})}=\left\Vert \int_{\mathsf{S}_{0}}q(0,x,T,\cdot
)\phi_{T}(\cdot)\,\nu_{0}(x)\,dx\right\Vert _{\mathcal{H}^{1,\alpha}%
(\overline{U}_{T})}\\
\leq\int_{\mathsf{S}_{0}}\left\Vert q(0,x,T,\cdot)\phi_{T}(\cdot)\,\right\Vert
_{\mathcal{H}^{1,\alpha}(\overline{U}_{T})}\nu_{0}(x)\,dx\leq B_{q}/Q_{\min}.
\end{multline*}
We have
\begin{eqnarray*}
\mathcal{E}_{T}^{N}[f]-\mathcal{E}_{T}[f]=\mathcal{E}_{T}^{N}[f]\frac{Q_{T}\phi_{0}-S_{N}[1_{\mathsf{S}_{T}}]}{Q_{T}\phi_{0}}+\frac{S_{N}[\rho_{T}f]-\phi_{0}\mathcal{E}_{T}[f]}{\phi_{0}},
\\
\mathcal{E}_{0}^{N}[f]-\mathcal{E}_{0}[f]=\mathcal{E}_{0}^{N}[f]\frac{Q_{0}\phi_{T}-\widetilde{S}_{N}[1_{\mathsf{S}_{0}}]}{Q_{0}\phi_{T}}+\frac{\widetilde{S}_{N}[\rho_{0}f]-\phi_{T}\mathcal{E}_{0}[f]}{\phi_{T}}.
\end{eqnarray*}
Hence, from the estimate (\ref{det_term_err}) and Corollary~\ref{cor:concentr} it follows that
\begin{eqnarray*}
\mathbb{E}\left[\left\Vert \mathcal{E}_{0}^{N}(g_{0,\ell})-\mathcal{E}_{0}(g_{0,\ell})\right\Vert _{\infty}|\mathcal{F}_\ell\right]&\lesssim & 
\frac{C}{\sqrt{N\delta^{d}}}\sqrt{\kappa^{2}(\mathcal{E}_{0})\|\phi_0\|_{\infty}}+B_q(\rho_{\max}/g^\star_{\min})\varkappa(\delta/2)^{1+\alpha},
\\
\mathbb{E}\left[\left\Vert \mathcal{E}_{T}%
^{N}(g_{1,\ell})-\mathcal{E}_{T}(g_{1,\ell})\right\Vert_{\infty}|\mathcal{F}_\ell\right]&\lesssim & \frac{C}{\sqrt{N\delta^{d}}}\sqrt{\kappa^{2}(\mathcal{E}_{T})\|\phi_T\|_{\infty}}+(B_q/Q_{\min})\varkappa(\delta/2)^{1+\alpha}
\end{eqnarray*}
with probability $1.$
As a result, under the choice \(\delta_{N} =  N^{-2/(2(1+\alpha)+d)} \)
we get
\begin{eqnarray*}
\mathbb{E}\left[\left\Vert \mathcal{E}_{0}^{N}(g_{0,\ell})-\mathcal{E}_{0}(g_{0,\ell})\right\Vert _{\infty}\right]\leq C_0 N^{-\frac{(1+\alpha)}{2(1+\alpha)+d}},\quad \mathbb{E}\left[\left\Vert \mathcal{E}_{T}^{N}(g_{1,\ell})-\mathcal{E}_{T}(g_{1,\ell})\right\Vert _{\infty}\right]\leq C_1 N^{-\frac{(1+\alpha)}{2(1+\alpha)+d}}
\end{eqnarray*}
where the constants \(C_1,C_2\) depend on \(q_{\min},q_{\max},\rho_{\min},\rho_{\max},B_q\) and \(K_\infty\).

\subsection{Proof of Theorem~\ref{thm:lb}}

We first note that since $[0,1]^{d}$ is regularly compact, we may simply take $\overline U_{0}$ $=$ $[0,1]^{d}$ and $\phi_0$ $=$ $\rho_0$ as sampling measure. 
%slightly different from  the $\phi_0$ introduced in Section~\ref{sec:aip}.
Let $Q$ be a continuous
strictly positive density function on $\mathbb{R}^{d}$ such that $Q(x-y)$ is a
transition kernel that satisfies Assumption ~\ref{ass:qrho}. Define
$K(x):=\exp\left(  -\frac{1}{1-x^{2}}\right)  1_{\left\{  -1\leq
x\leq1\right\}  }.$ Note that $K$ is infinitely smooth on the real line, and
all its derivatives vanish outside of $(-1,1).$ Set $\beta=1+\alpha$ and
\begin{align*}
\psi_{1}(x) &  :=1-L_{1}h^{\beta+d/2}\Psi(x/h),\\
\psi_{2}(y) &  :=1-L_{2}h^{\beta}\Psi(y/h)
\end{align*}
for some $L_{1,2}\in(0,1),$ $0<h<1,$ where
\[
\Psi(z):=\frac{K^{\otimes d}(z)}{\Vert K^{\otimes d}\Vert_{\mathcal{H}%
^{1,\alpha}([0,1]^{d})}}\text{ \ \ with \ \ }K^{\otimes d}(z):=\prod
_{i=1}^{d}K(z_{i}),\quad z=(z_{1},\ldots,z_{d})\in\mathbb{R}^{d}.
\]
Furthermore, let $Q_{\psi}(x,y):=\xi(x)\psi_{1}(x)\psi_{2}(y)Q(x-y)$ be
transition density with $\xi>0$ being a normalization factor. It is clear that
$Q_{\psi}$ satisfies Assumption ~\ref{ass:qrho} also. Let $g_{\psi}$ be the
unique solution of the fixed point problem
\[
\int_{\lbrack0,1]^{d}}\frac{\rho_{0}(x)}{\int_{[0,1]^{d}}Q_{\psi}%
(x,z)\frac{\rho_{T}(z)}{g_{\psi}(z)}\,dz}Q_{\psi}(x,y)\,dx=g_{\psi}(y)
\]
with $\int_{\lbrack0,1]^{d}} g_{\psi}(z)\,dz=1$. Then we have%
\[
\int_{\lbrack0,1]^{d}}\frac{\rho_{0}(x)}{\int_{[0,1]^{d}}Q(x-z)\frac{\rho
_{T}(z)}{g_{\psi}(z)/\psi_{2}(z)}\,dz}Q(x-y)dx=g_{\psi}(y)/\psi_{2}(y)
\]
So if $g_{1}\geq0$ is the unique solution of the equation
\[
\int_{\lbrack0,1]^{d}}\frac{\rho_{0}(x)}{\int_{[0,1]^{d}}Q(x-z)\frac{\rho
_{T}(z)}{g_{1}(z)}\,dz}Q(x-y)\,dx=g_{1}(y)
\]
satisfying $\int_{\lbrack0,1]^{d}} g_{1}(z)\,dz=1,$ we have
\begin{align*}
g_{\psi}/\psi_{2} &  =\theta g_{1}\text{ \ \ }\\
\text{with }\theta &  =\int_{[0,1]^{d}}\frac{g_{\psi}(z)}{\psi_{2}%
(z)}dz=\left(  \int_{[0,1]^{d}}g_{1}(z)\psi_{2}(z)dz\right)  ^{-1}.
\end{align*}
Note that both functions $g_{1}$ and $g_{\psi}$ are bounded from below and
above by positive constants for $0<h<1$, see \eqref{eq:prior-g}.
In particular, 
\begin{equation}
    g_1(z),g_\psi(z) \in [g_{\min}, g_{\max}] \quad \text{ for all } \quad  z \in [0, 1]^d \label{eq:g1-gpsi-bounds}
\end{equation}
with $0 < g_{\min} \leq g_{\max}.$
Then we have for $h$ small enough,
$$
C_{\min} h^{\beta + d} \leq \theta - 1 \leq \frac{C_{\max} h^{\beta + d}}{1 - C_I h^{\beta + d}}
$$
with
$$
C_{\min} := L_2 g_{\min} \int_{[0, 1]^d} \Psi(u)\, du, \quad C_{\max} := L_2 g_{\max} \int_{[0, 1]^d} \Psi(u)\, du, \quad C_I := L_2  g_{\max}.
$$

 Denote by
$P_{\psi}$ the distribution of $(X_{0},X_{T})$ under $Q_{\psi},$ that is,
$X_{0}\sim\rho_{0}$ and $X_{T}|X_{0}\sim Q_{\psi}.$ Due to Remark~\ref{remHoel} we
have for all $\boldsymbol{\gamma}\in\mathbb{N}^{d},$ with $\left\vert
\boldsymbol{\gamma}\right\vert =1,$
\[
\left\vert D^{\boldsymbol{\gamma}}\Psi(x)-D^{\boldsymbol{\gamma}}%
\Psi(y)\right\vert \leq\Vert x-y\Vert^{\alpha},\quad x,y\in [0,1]^{d}.
\]

Hence 
\begin{align*}
\mathrm{KL}(P_{\psi}^{\otimes N}\|P_{1}^{\otimes N}) &=N\mathrm{KL}(P_{\psi}\|P_{1})  \\
 & = N\int\int\rho_{0}(x)Q(x-y)\log\left(\frac{\rho_0(x) Q(x-y)}{\rho_0(x)\xi(x)\psi_{1}(x)\psi_{2}(y)Q(x-y)}\right)\,dx\,dy
 \\
 & \leq N\int\int\rho_{0}(x)Q(x-y)\left(\xi(x)\psi_{1}(x)\psi_{2}(y)-1\right)^{2}\,dx\,dy\\
 & \lesssim  Nh^{2\beta}\int\int\rho_{0}(x)Q(x-y)\Psi^2(y/h)\,dx\,dy
\\
 &
 +N\int\int\rho_{0}(x)Q(x-y)(1-\xi(x))^2\,dx\,dy
 \\
 & \lesssim Nh^{2\beta+d}
\end{align*}
since 
\[
\xi(x)=\frac{1}{\psi_{1}(x)\int_{\lbrack0,1]^{d}}\psi_{2}(y)Q(x-y)\,dy}=1+O(h^{\beta+d/2}).
\]
Moreover, we obviously have
\[
\psi_{2}(0)=1-L_{2}h^{\beta}\Psi(0)
\]
and
\[
(g_{1}(0)-g_{\psi}(0))/g_{1}(0)=1-\psi_2(0)\theta\geq ch^{\beta}
\]
for some $c>0.$
Using the bounds \eqref{eq:g1-gpsi-bounds}, Lemma~\ref{HmTr1} and Lemma~\ref{lemH}  (note that $\int g_{\psi}(x)\,dx=\int g_{1}(x)\,dx=1$), we derive
\[
d_H(g_{\psi},g_{1})\gtrsim h^{\beta}\Psi(0).
\]
We are now ready to apply Assouad's lemma in the Kullback-Leibler version with $h=N^{-1/(2\beta+d)}$, see Theorem ~2.2 in \cite{tsybakov2003introduction}. As a result, we derive \eqref{eq:lower-bound}.
\appendix

\section{Reciprocal  processes}
\label{sec:rp}
Let $X\equiv\left(  X_{t}\right)  _{t\geq0}$ be a stochastic process on a
probability space $(  \Omega,\mathcal{F},\left(  \mathcal{F}_{t})
_{t\geq0},\mathbb{P}\right)  $ with state space $\mathbb{R}^{d}$. It is
assumed that the filtration $( \mathcal{F}_{t})  $ is generated by
the trajectories of $X$ in the usual way, and that the dynamics of $X$ are
governed by non-zero transition densities
\begin{equation}
q(s,x;t,y),\text{ }0\leq s<t,\text{ }x,y\in\mathbb{R}^{d} \label{pdens}%
\end{equation}
that satisfy the Chapman-Kolmogorov equation%
\begin{equation}
q(s,x;t,y)=\int_{\mathbb{R}^{d}}q(s,x;t^{\prime},y^{\prime})q(t^{\prime
},y^{\prime};t,y)dy^{\prime},\quad 0\leq s<t,\quad x,y\in\mathbb{R}^{d}.
\label{schroe0}%
\end{equation}
Let us fix a terminal time $T>0$ and consider the \textquotedblleft
intermediate transition densities\textquotedblright%
\begin{equation}
p(s,x;t,y;T,z)=\frac{q(s,x;t,y)q(t,y;T,z)}{q(s,x;T,z)},\ \ 0\leq s<t<T,\text{
}x,y,z\in\mathbb{R}^{d}, \label{qdens}%
\end{equation}
and a given probability distribution $\mu(dx,dz)$ on $\mathbb{R}^{d}%
\times\mathbb{R}^{d}$ with marginals $\rho_{0}(dx)=\mu(dx,\mathbb{R}^{d})$ and
$\rho_{T}(dz)=\mu(\mathbb{R}^{d},dz),$ respectively. It is not difficult to
check that the system (\ref{qdens}) satisfies the Chapman-Kolmogorov equation
for each fixed $z\in\mathbb{R}^{d}.$ Due to \cite{Jam74} there exists a
process $X^{\mu}\equiv\left(  X_{t}^{\mu}\right)  _{0\leq t\leq T}$ with
finite dimensional distributions characterized by%
\begin{multline}
\mathbb{E}\left[  g(X_{0}^{\mu},X_{t_{1}}^{\mu},\ldots,X_{t_{n}}^{\mu}%
,X_{T}^{\mu})\right]  =\int_{\mathbb{R}^{d}\times\mathbb{R}^{d}}%
\mu(dx,dz)\cdot\label{eq:schroe-findim}\\
\int_{\left(  \mathbb{R}^{d}\right)  ^{n}}dx_{1}p(0,x;t_{1},x_{1};T,z)\cdots
dx_{n}p(t_{n-1},x_{n-1};t_{n},x_{n}T,z)g(x,x_{1},\ldots,x_{n},z)\\
=\int_{\mathbb{R}^{d}\times\mathbb{R}^{d}}\mu(dx,dz)\mathbb{E}\left[  \left.
g\left(  x,X_{t_{1}}^{x},\ldots,X_{t_n}^{x},z\right)  \right\vert X_{T}%
^{x}=z\right]  ,
\end{multline}
for any grid $0<t_{1}<\ldots<t_{n}<T$, non-negative Borel test function
$g:\left(  \mathbb{R}^{d}\right)  ^{n+2}\rightarrow\mathbb{R}_{\geq0},$ and
$X^{x}$ denoting the initial process starting in $X_{0}^{x}=x.$ In particular,
for $n=0$ one has that%
\begin{equation}
\mathbb{E}\left[  g(X_{0}^{\mu},X_{T}^{\mu})\right]  =\int_{\mathbb{R}%
^{d}\times\mathbb{R}^{d}}\mu(dx,dz)g(x,z). \label{endpdis}%
\end{equation}
Furthermore, in \cite{Jam74} it is shown that $X^{\mu}$ is a
\textit{reciprocal process}, i.e. it satisfies for any $0\leq s<t\leq T,$%
\[
\mathbb{P}\left(  A\cap B|X_{s},X_{t}\right)  =\mathbb{P}\left(  A|X_{s}%
,X_{t}\right)  \mathbb{P}\left(  B|X_{s},X_{t}\right)  ,
\]
if $A\in\sigma\left(  X_{r}:0\leq r<s\right)  $ or $A\in\sigma\left(
X_{r}:t<r\leq T\right)  ,$ and $B\in\sigma\left(  X_{r}:s<r<t\right)  .$ In
general, any Markov process is reciprocal but not necessarily the other way
around. Due to \cite{Jam74} the process $X^{\mu}$ is Markov if and only if
there exist $\sigma$-finite measures $\nu_{0}$ and $\nu_{T}$ on $\mathbb{R}%
^{d}$ such that
\begin{equation}
\mu\left(  dx,dz\right)  =q(0,x;T,z)\nu_{0}(dx)\nu_{T}(dz). \label{schroe2}%
\end{equation}
If (\ref{schroe2}) applies, the $X^{\mu}$ is called the Markov process of Schr\"{o}dinger.

\section{Reverse process in diffusion setting}

\label{subsec:reverse_diffusion_construction}

Let us consider the SDE%
\begin{equation}
dX_{s}=a(s,X_{s})ds+\sigma(s,X_{s})dW_{s},\text{ \ \ }0\leq s\leq T,
\label{SDE}%
\end{equation}
where $X\in\mathbb{R}^{d},$ $a:\left[  0,T\right]  \times\mathbb{R}%
^{d}\rightarrow\mathbb{R}^{d},$ $\sigma:\left[  0,T\right]  \times
\mathbb{R}^{d}\rightarrow\mathbb{R}^{d\times m},$ and $W$ is an $m$%
-dimensional standard Wiener process. We assume that the coefficients of
(\ref{SDE}) are $C^{\infty}$ with bounded derivatives of any order, and such
that $X$ is governed by a $C^{\infty}$ transition density (\ref{pdens}) that
satisfies (\ref{schroe0}). Let us recall the construction in \cite{MSS2004} of
an $\mathbb{R}^{d+1}$-valued so called \textquotedblleft
reverse\textquotedblright\ process
\begin{equation}
\left(  Y_{s}^{y},\mathcal{Y}_{s}^{y}\right)  _{0\leq s\leq T},\text{
\ \ \ }y\in\mathbb{R}^{d}, \label{revproc}%
\end{equation}
that allows for a stochastic representation%
\begin{equation}
\int q(0,x;T,y)g(x)\,dx=\mathbb{E}[g(Y_{T}^{y})\mathcal{Y}_{T}^{y}],\text{
\ \ \ }y\in\mathbb{R}^{d},\text{ }T>0, \label{revrepresent}%
\end{equation}
for any Borel (test) function $g:\mathbb{R}^{d}\rightarrow\mathbb{R}_{\geq0}.$
In \cite{MSS2004} it is shown that (\ref{revrepresent}) holds for a process
(\ref{revproc}) that solves the SDE%
\begin{align}
dY_{s}  &  =\alpha\left(  s,Y_{s}\right)  ds+\widetilde{\sigma}\left(
s,Y_{s}\right)  d\widetilde{W}_{s},\text{ \ \ }Y_{0}=y,\nonumber\\
\mathcal{Y}_{s}  &  =\exp\left(  \int_{0}^{s}c(u,Y_{u})du\right)  ,
\label{Yrev}%
\end{align}
with $\widetilde{W}$ being an independent copy of $W,$ and%
\begin{align*}
\alpha^{i}\left(  s,y\right)   &  :=\sum_{j=1}^{d}\frac{\partial}{\partial
y^{j}}b^{ij}\left(  T-s,y\right)  -a^{i}\left(  T-s,y\right)  ,\text{
\ \ \ \ \ }b:=\sigma\sigma^{\top},\\
\widetilde{\sigma}\left(  s,y\right)   &  :=\sigma\left(  T-s,y\right)  ,\\
c(s,y)  &  :=\frac{1}{2}\sum_{i,j=1}^{d}\frac{\partial^{2}}{\partial
y^{i}\partial y^{j}}b^{ij}\left(  T-s,y\right)  -\sum_{i=1}^{d}\frac{\partial
}{\partial y^{i}}a^{i}\left(  T-s,y\right)  .
\end{align*}
For technical details we refer to \cite{MSS2004}. Essentially, the idea behind
a reverse diffusion in the above sense goes back to \cite{Thomson1987} (see
also \cite{KurSabRan} for example).

\section{Forward-Reverse approach for conditional diffusions}

\label{RecapFR}

In \cite{MSS2004}, the reverse process (\ref{Yrev}) served as a corner stone
for the construction of a \emph{forward-reverse }(FR) density estimator for
the density $q(0,x,y,T)$ with root-N consistency. In \cite{BS2014}, this
forward-reverse estimation approach was extended to conditional diffusions (or
diffusion bridges), in order to estimate generically the finite dimensional
distributions of a conditional diffusion. We here summarize the main results
of \cite{BS2014}.  

\begin{theorem}
\cite[Thm. 3.4]{BS2014} Consider a time grid%
\[
0=s_{0}<s_{1}<\cdots<s_{K}=t^{\ast}=t_{0}<t_{1}<\cdots<t_{L}=T
\]
and define%
\[
\widehat{t}_{i}:=t_{L}-t_{L-i}=T-t_{L-i},\text{ \ \ }i=1,\ldots,L.
\]
Let%
\[
K_{\varepsilon}(u):=\varepsilon^{-d}K(u/\varepsilon),\text{ \ \ }%
u\in\mathbb{R}^{d},
\]
where $K$ is integrable with $\int_{\mathbb{R}^{d}}K(u)du=1$ 
and   $\int_{\mathbb{R}^{d}}u_iK(u)du=0$ for $i=1,\ldots,d.$
Let $X^{x}$
satisfy (\ref{SDE}) with $X_{0}^{x}=x\in\mathbb{R}^{d}$, and let
$y\in\mathbb{R}^{d}$. For any bounded measurable $g:\left(  \mathbb{R}%
^{d}\right)  ^{(K+L-1)}\rightarrow\mathbb{R}$ we define the functional%
\[
\mathcal{E}(g;x,y):=\mathbb{E}\left[  \left.  g\left(  X_{s_{1}}^{x}%
,\ldots,X_{s_{K-1}}^{x},X_{t^{\ast}}^{x},X_{t_{1}}^{x},\ldots,X_{t_{L-1}}%
^{x}\right)  \right\vert X_{T}^{x}=y\right]  ,
\]
and for $\varepsilon>0$ the stochastic representation%
\begin{multline}
H_{\varepsilon}(g;x,y):=\mathbb{E}\left[  g\left(  X_{s_{1}}^{x}%
,\ldots,X_{s_{K-1}}^{x},X_{t^{\ast}}^{x},Y_{\widehat{t}_{L-1}}^{y}%
,\ldots,Y_{\widehat{t}_{1}}^{y}\right)  \right. \label{heps}\\
\left.  \times K_{\varepsilon}(Y_{T-t^{\ast}}^{y}-X_{t^{\ast}}^{x}%
)\mathcal{Y}_{T-t^{\ast}}^{y}\right]  .
\end{multline}
One then has%
\begin{equation}
\mathcal{E}(g;x,y)q(0,x,T,y)=H(g;x,y):=\lim_{\varepsilon\downarrow
0}H_{\varepsilon}(g;x,y). \label{FRrep}%
\end{equation}

\end{theorem}

In \cite{BS2014} a Monte Carlo procedure for estimating (\ref{FRrep}) is
proposed and analysed: Consider the Monte Carlo estimator%
\begin{multline*}
\widehat{H}_{\varepsilon,M,N}(g;x,y):=\\
\frac{1}{NM}\sum_{n=1}^{N}\sum_{m=1}^{M}g\left(  X_{s_{1}}^{x,n}%
,\ldots,X_{s_{K-1}}^{x,n},X_{t^{\ast}}^{x,n},Y_{\widehat{t}_{L-1}}%
^{y,m},\ldots,Y_{\widehat{t}_{1}}^{y,m}\right)  K_{\varepsilon}(Y_{T-t^{\ast}%
}^{y,m}-X_{t^{\ast}}^{x,n})\mathcal{Y}_{T-t^{\ast}}^{y,m}%
\end{multline*}
corresponding to (\ref{heps}), where the superscripts $m$ and $n$ denote
independently simulated trajectories of the corresponding processes. We recall
\cite[Thm. 3.4]{BS2014}:

\begin{theorem}
\label{FRconv} Assume conditions~\cite[4.1, 4.4, and 4.5]{BS2014} and set
$M=N$ and $\varepsilon=\varepsilon_{N}$ depending on $N$. One then has for
fixed $x,y\in\mathbb{R}^{d}$:

\begin{itemize}
\item If $d\leq4$ and $\varepsilon_{N}=CN^{-\alpha}$ for some $1/4\leq
\alpha\leq1/d$ one has that \\ $\mathbb{E}\left[  \left(  \widehat{H}%
_{\varepsilon_{N},N,N}(g;x,y)-H(g;x,y)\right)  ^{2}\right] $ $ =$ $\mathcal{O}%
(N^{-1})$, hence the optimal convergence rate $1/2$.

\item If $d>4$ and $\varepsilon_{N}=CN^{-2/(4+d)}$ one obtains\newline%
$\mathbb{E}\left[  \left(  \widehat{H}_{\varepsilon_{N},N,N}%
(g;x,y)-H(g;x,y)\right)  ^{2}\right]  =\mathcal{O}(N^{-8/(4+d)})$.\bigskip
\end{itemize}
\end{theorem}

Hence, in particular, for a second order kernel $K$ and $d\leq4$, both
$H(g;x,y)$ and $H(1;x,y)=q(0,x,T,y)$ in (\ref{FRrep}) may be approximated with
$1/\sqrt{N}$ accuracy by using $N$ forward trajectories of $X$ and $N$
\textquotedblleft reverse\textquotedblright\ trajectories of $(Y,\mathcal{Y}%
)$. One so may obtain an estimate for $\mathcal{E}(g)$ by the ratio of these
respective approximations.

\section{Hilbert metric}\label{recapHM}

Let $\mathcal{L}_{+}^{\infty}(\mathsf{S})$ denote the set of (equivalence classes of) strictly positive measurable functions on $\mathsf{S}$ that are essentially bounded away from zero.  
For two such functions $f,g \in \mathcal{L}_{+}^{\infty}(\mathsf{S})$, define
\begin{eqnarray*}
  M(f,g) 
  &:= &\inf\{\lambda>0 : f \le \lambda\,g\text{ a.e.}\}, 
  \\
  m(f,g) 
  &:=& \sup\{\lambda>0 : \lambda\,g \le f\text{ a.e.}\},
\end{eqnarray*}
and
\begin{eqnarray*}
d_{H}(f,g) 
  := \log \Bigl(\tfrac{M(f,g)}{m(f,g)}\Bigr)
\end{eqnarray*}
with  \(\inf \varnothing = +\infty,\) by definition.
If $m(f,g) = 0$ or $M(f,g) = +\infty$, then by convention $d_H(f,g) = +\infty$.  
Moreover, since $d_{H}(\alpha f,\alpha g)=d_{H}(f,g)$ for any $\alpha>0$, 
one regards $d_{H}$ as a metric on the equivalence classes of functions 
generated by the relation
\[
  f \,\sim\, g 
  \;\Longleftrightarrow\; 
  \frac{f}{g} \equiv \text{constant a.e.}
\]
Equivalently, $d_{H}(f,g) = 0$ if and only if $f \sim g$.

\begin{lemma}
\label{lemH} For $f,g\in\mathcal{L}_{+}^{\infty}(\mathsf{S})$  one has that%
\[
 d_{H}(f,g)\leq\frac{2}{\min(\inf f,\inf g)}\left\Vert f-g\right\Vert _{\infty}.
\]
Moreover, if $\sup(f/g)\geq 1$  and $\inf(f/g)\leq 1$ we have
\[
d_{H}(f,g)\geq \frac{\min(\inf f,\inf g)}{\|f\|_{\infty}\|g\|_{\infty}}\left\Vert f-g\right\Vert _{\infty}. 
\]
\end{lemma}

\begin{proof}
Note that with%
\begin{align*}
M(f,g) &  =\inf\left\{  \lambda:f\leq\lambda g\right\}  =\sup\left(
f/g\right)  \\
m(f,g) &  =\sup\left\{  \lambda:\lambda g\leq f\right\}  =\inf\left(
f/g\right)
\end{align*}
one has%
\begin{align*}
d_{H}(f,g) &  =\log\frac{\sup\left(  f/g\right)  }{\inf\left(  f/g\right)  }\\
&  =\sup\log\left(  f/g\right)  +\sup\log\left(  g/f\right)  \\
&  =\sup\left(  \log f-\log g\right)  +\sup\left(  \log g-\log f\right)  ,
\end{align*}
hence%
\[
d_{H}(f,g)\leq2\left\Vert \log f-\log g\right\Vert _{\infty}.
\]
Next consider that%
\begin{align*}
\left\vert \log f(x)-\log g(x)\right\vert  &  \leq\frac{1}{\min(f(x),g(x))}%
\left\vert f(x)-g(x)\right\vert \\
&  \leq\frac{1}{\min(\inf f,\inf g)}\left\vert f(x)-g(x)\right\vert .
\end{align*}
To prove the second inequality, under the given conditions, note that 
\begin{equation}
\log\sup(f/g)-\log\inf(f/g)\geq\max(\log\sup(f/g),\log\sup(g/f)).\label{maxeq}%
\end{equation}
Let us consider the case $\sup(f/g)\geq\sup(g/f)\geq1.$Then using the
elementary inequality $\log(1+x)$ $\geq x/(1+x),$ $x\geq0,$ we derive from
(\ref{maxeq}),
\begin{align}
d_{H}(f,g)  & \geq\log\sup(f/g)\geq\frac{\sup(f/g)-1}{\sup(f/g)}\nonumber\\
& \geq\frac{\inf g}{\Vert f\Vert_{\infty}}\left(  \sup(f/g)-1\right)
\nonumber\\
& \geq\frac{\inf g}{\Vert f\Vert_{\infty}\Vert g\Vert_{\infty}}\Vert
f-g\Vert_{\infty},\label{Lem4eq}%
\end{align}
since $\Vert f-g\Vert_{\infty}\leq\Vert g\Vert_{\infty}(\sup(f/g)-1).$ For the
case $1\leq\sup(f/g)\leq\sup(g/f)$ we may exchange $f$ and $g$ in
(\ref{Lem4eq}), and we are done.
\end{proof}

\begin{lemma}\label{Dhineq}
Let $f,g:\mathsf{S}\rightarrow\mathbb{R}_{>0}$, $\mathsf{S}\subset\mathbb{R}^{d},$ be bounded,
and bounded away from zero. Let further $g\in\left[  a,b\right]  $ with
$0<a<b.$ Then, under the condition%
\begin{equation}
\sup\frac{f}{g}1_{a\leq f\leq b}\geq1\text{ \ \ and \ }\inf\frac{f}{g}1_{a\leq
f\leq b}\leq1,\label{Hmcon}%
\end{equation}
it holds that%
\begin{equation}
d_{H}(\mathcal{T}_{\left[  a,b\right]  }f,g)\leq d_{H}(f,g),\label{Hmineq}%
\end{equation}
where $\mathcal{T}_{\left[  a,b\right]  }$ is the truncation operator defined
in (\ref{eq:picar-iter}). 
\end{lemma}

\begin{proof}
We consider the following cases:\ (I) $\left\{  f<a\right\}  \neq\varnothing$
and$\left\{  f>b\right\}  \neq\varnothing:$ We then have%
\begin{align*}
\sup\frac{\mathcal{T}_{\left[  a,b\right]  }f}{g} &  =\max\left(  \sup\frac
{a}{g}1_{f<a},\sup\frac{f}{g}1_{a\leq f\leq b},\sup\frac{b}{g}1_{f>b}\right)
\\
&  \leq\max\left(  \sup\frac{f}{g}1_{f<a},\sup\frac{f}{g}1_{a\leq f\leq
b},\sup\frac{f}{g}1_{f>b}\right)  =\sup\frac{f}{g}\text{ \ \ and}%
\end{align*}%
\begin{align*}
\inf\frac{\mathcal{T}_{\left[  a,b\right]  }f}{g} &  =\min\left(  \inf\frac
{a}{g}1_{f<a},\inf\frac{f}{g}1_{a\leq f\leq b},\inf\frac{b}{g}1_{f>b}\right)
\\
&  \geq\min\left(  \inf\frac{f}{g}1_{f<a},\inf\frac{f}{g}1_{a\leq f\leq
b},\inf\frac{f}{g}1_{f>b}\right)  =\inf\frac{f}{g},
\end{align*}
whence (\ref{Hmineq}) by Lemma~\ref{lemH}. (II) Case $\left\{  f<a\right\}  =\varnothing$ and
$\left\{  f>b\right\}  \neq\varnothing:$ We then have%

\begin{align*}
\sup\frac{\mathcal{T}_{\left[  a,b\right]  }f}{g} &  =\max\left(  \sup\frac
{f}{g}1_{a\leq f\leq b},\sup\frac{b}{g}1_{f>b}\right)  \\
&  \leq\max\left(  \sup\frac{f}{g}1_{a\leq f\leq b},\sup\frac{f}{g}%
1_{f>b}\right)  =\sup\frac{f}{g}\text{ \ \ and}%
\end{align*}
\begin{align*}
\inf\frac{\mathcal{T}_{\left[  a,b\right]  }f}{g} &  =\min\left(
\underset{\in\left[  a/b,b/a\right]  }{\underbrace{\inf\frac{f}{g}1_{a\leq
f\leq b}}},\underset{\in\left[  1,b/a\right]  }{\underbrace{\inf\frac{b}%
{g}1_{f>b}}}\right)  \\
&  \geq\min\left(  \underset{\in\left[  a/b,b/a\right]  }{\underbrace{\inf
\frac{f}{g}1_{a\leq f\leq b}}},\underset{\in\lbrack1,\infty)}{\underbrace{\inf
\frac{f}{g}1_{f>b}}}\right)  =\inf\frac{f}{g}%
\end{align*}
due to condition (\ref{Hmcon}), which yields (\ref{Hmineq}). (III)\ Case
$\left\{  f<a\right\}  \neq\varnothing$ and $\left\{  f>b\right\}
=\varnothing:$ We then have%
\begin{align*}
\sup\frac{\mathcal{T}_{\left[  a,b\right]  }f}{g} &  =\max\left(
\underset{\in\left[  1,a/b\right]  }{\underbrace{\sup\frac{a}{g}1_{f<a}}%
},\underset{\in\left[  a/b,b/a\right]  }{\underbrace{\sup\frac{f}{g}1_{a\leq
f\leq b}}}\right)  \\
&  \leq\max\left(  \underset{\in(0,1]}{\underbrace{\sup\frac{f}{g}1_{f<a}}%
},\underset{\in\left[  a/b,b/a\right]  }{\underbrace{\sup\frac{f}{g}1_{a\leq
f\leq b}}}\right)  =\sup\frac{f}{g}%
\end{align*}
due to condition (\ref{Hmcon}), and%

\begin{align*}
\inf\frac{\mathcal{T}_{\left[  a,b\right]  }f}{g} &  =\min\left(  \inf\frac
{a}{g}1_{f<a},\inf\frac{f}{g}1_{a\leq f\leq b}\right)  \\
&  \geq\min\left(  \inf\frac{f}{g}1_{f<a},\inf\frac{f}{g}1_{a\leq f\leq
b}\right)  ,
\end{align*}
which yields (\ref{Hmineq}). (IV) Case $a\leq f\leq b:$ Then $\mathcal{T}%
_{\left[  a,b\right]  }f=f$ by construction of $\mathcal{T}_{\left[
a,b\right]  }.$
\end{proof}

\begin{lemma}
\label{HmTr1} Let $\mathsf{S}\subset\mathbb{R}^{d}$ be a connected compact set
with positive Lebesgue measure. Let $f,g:\mathsf{S}\rightarrow\mathbb{R}_{>0}$
be bounded and continuous, and
\begin{equation}
\int_{\mathsf{S}}f^{p}(x)dx=\int_{\mathsf{S}}g^{p}(x)dx\text{ \ for some
}p\geq1.\label{normfg}%
\end{equation}
Then there exists $x_{0}\in\mathsf{S}$ with $f(x_{0})/g(x_{0})=1$.
\end{lemma}

\begin{proof}
By (\ref{normfg}) one has $\sup_{\mathsf{S}}\left(  f^{p}-g^{p}\right)  \geq0$
and $\inf_{\mathsf{S}}\left(  f^{p}-g^{p}\right)  \leq0$ since $\mathsf{S}$
has positive Lebesgue measure. This implies by continuity that there exist
$x_{+},x_{-}\in\mathsf{S}$ such that $\sup_{\mathsf{S}}\left(  f^{p}%
-g^{p}\right)  =f^p(x_{+})-g^p(x_{+})\geq0$ and $\inf_{\mathsf{S}}\left(
f^{p}-g^{p}\right)  =f^p(x_{-})-g^p(x_{-})\leq0,$ respectively. If $f(x_{+}%
)-g(x_{+})=0$ or $f(x_{-})-g(x_{-})=0$ we may take $x_{0}=x_{+}$ or
$x_{0}=x_{-}$, respectively,  since $f,g>0.$ If $f(x_{+})-g(x_{+})>0$ and $f(x_{-}%
)-g(x_{-})<0$ there exists $x_{0}\in\mathsf{S}$ with $f(x_{0})-g(x_{0})=0$
since $\mathsf{S}$ is connected, hence $f(x_{0})/g(x_{0})=1.$
\end{proof}

\begin{corollary}\label{HmTr}
If $f$ and $g$ and $\mathsf{S}$ in Lemma~\ref{Dhineq} satisfy in addition the conditions of Lemma~\ref{HmTr1}, then (\ref{Hmcon}) is satisfied and thus (\ref{Hmineq}) holds.
\end{corollary}

\section{Smoothness classes and H\"{o}lder spaces}\label{SmoothC}

In this section we recall some classical terminology on smoothness
classifications from standard analysis. For an open domain $U\subset
\mathbb{R}^{d}$, the space $C^{k}(U)$, $k\in\mathbb{N}_{0},$ denotes the space
of functions $f:U\rightarrow\mathbb{R}$ that have continuous partial
derivatives up to order $k$. Formally, these derivatives are described by a
differential operator $D^{{\boldsymbol{\gamma}}}$ of order $\left\vert
{\boldsymbol{\gamma}}\right\vert ,$ where ${\boldsymbol{\gamma}}=(\gamma
_{1},\dots,\gamma_{d})\in\mathbb{N}_{0}^{d}$ is a multi-index, $\left\vert
{\boldsymbol{\gamma}}\right\vert =\sum_{i=1}^{d}\gamma_{i},$ and
\[
D^{{\boldsymbol{\gamma}}}f:=\frac{\partial^{|{\boldsymbol{\gamma}}|}%
f}{\partial x_{1}^{\gamma_{1}}\cdots\partial x_{d}^{\gamma_{d}}},\text{ \ for
\ \ }f\in C^{k}(U).
\]
For clarity, multi-indices are distinguished by using bold font. If $U$ is bounded, we denote by $C^{k}(\overline{U})$ the space of
uniformly continuous functions $f:U\rightarrow\mathbb{R}$ with uniformly
continuous partial derivatives up to order $k$. Hence, if $f\in C^{k}%
(\overline{U})$ then $f$ and all his partial derivatives extend to continuous
functions on $\overline{U}.$

A function $f:U\rightarrow\mathbb{R}$ is said to be uniformly
$\alpha$-H\"{o}lder continuous with exponent $\alpha$ for some $0<\alpha\leq
1$, if
\[
\lbrack f]_{\alpha,U}:=\sup_{x\neq y\in U}\frac{|f(x)-f(y)|}{|x-y|^{\alpha}%
}<\infty
\]
with $\left\vert \cdot\right\vert $ being a particularly chosen norm on
$\mathbb{R}^{d}.$ The function $f$ is said to be locally uniformly $\alpha
$-H\"{o}lder continuous, if $[f]_{\alpha,U^{\prime}}<\infty$ for any bounded
open set $U^{\prime}$ with $\overline{U^{\prime}}\subset U,$ i.e. for any open
set $U^{\prime}$ that is compactly contained in $U.$ The space of locally
uniformly $\alpha$-H\"{o}lder continuous functions in $U$ is denoted by
$\mathcal{H}^{0,\alpha}(U).$ If $U$ is bounded, we denote by $\mathcal{H}%
^{0,\alpha}(\overline{U})$ the space of uniformly $\alpha$-H\"{o}lder
continuous functions in $U.$ Note that any uniformly $\alpha$-H\"{o}lder
continuous functions in $U$ extends to an $\alpha$-H\"{o}lder continuous
function in $\overline{U}.$

We next define for an open domain $U\subset\mathbb{R}^{d}$ and $k\in
\mathbb{N}_{0}$ the space%
\[
\mathcal{H}^{k,\alpha}(U):=\left\{  f:U\rightarrow\mathbb{R}:f\in
C^{k}(U)\text{ and }D^{{\boldsymbol{\gamma}}}f\in\mathcal{H}^{0,\alpha
}(U)\text{ for all }{\boldsymbol{\gamma}}\text{ with }\left\vert
{\boldsymbol{\gamma}}\right\vert =k\right\}  ,
\]
and for open and bounded $U,$ the space $\mathcal{H}^{k,\alpha}(\overline{U})$
is defined as%
\[
\mathcal{H}^{k,\alpha}(\overline{U}):=\left\{  f:\Omega\rightarrow
\mathbb{R}:f\in C^{k}(\overline{U})\text{ and }D^{{\boldsymbol{\gamma}}}%
f\in\mathcal{H}^{0,\alpha}(\overline{U})\text{ for all }{\boldsymbol{\gamma}%
}\text{ with }\left\vert {\boldsymbol{\gamma}}\right\vert =k\right\}  .
\]
It is well known that $\mathcal{H}^{k,\alpha}(\overline{U})$ is a Banach space
with norm%
\[
\Vert f\Vert_{\mathcal{H}^{k,\alpha}(\overline{U})}=\max\left(  \max
_{\left\vert {\boldsymbol{\gamma}}\right\vert \leq k}\left\Vert
D^{{\boldsymbol{\gamma}}}f\right\Vert _{\infty,U},\max_{\left\vert
{\boldsymbol{\gamma}}\right\vert =k}[D^{{\boldsymbol{\gamma}}}f]_{\alpha
,U}\right)  .
\]

\begin{remark}\label{remHoel}
(i) For open and bounded $U,$ let $f\in\mathcal{H}^{k,\alpha}(\overline{U})$ for some fixed
$k\in\mathbb{N}_{0}$ and $0<\alpha\leq1.$ One then has for all
${\boldsymbol{\gamma}}$ with $\left\vert {\boldsymbol{\gamma}}\right\vert \leq
k$ and all $x,y\in U,$%
\[
\left\vert D^{{\boldsymbol{\gamma}}}f\left(  x\right)  -D^{{\boldsymbol{\gamma
}}}f\left(  y\right)  \right\vert \leq c_{\left\vert {\boldsymbol{\gamma}}\right\vert,d}\Vert f\Vert_{\mathcal{H}%
^{k,\alpha}(\overline{U})}\,|x-y|^{\beta}\text{ \ \ with }\beta=\left\{
\begin{array}
[c]{c}%
1\text{ \ if }\left\vert {\boldsymbol{\gamma}}\right\vert <k\\
\alpha\text{ \ if }\left\vert {\boldsymbol{\gamma}}\right\vert =k
\end{array}
\right.  \text{ ,}%
\]
where $c_{k,d}=1,$ and if $0\leq l<k$, $c_{l,d}$ only depends on the particular norm that is chosen on
$\mathbb{R}^{d}.$ For example, if $|\cdot|$ denotes the max-norm one may take
$c_{l,d}=
%\sqrt
{d}$  for $l <k$.

(ii) The above terminology obviously extends to spaces of vector functions
$f:U\subset\mathbb{R}^{d}\rightarrow\mathbb{R}^{m},$ for arbitrary
$m\in\mathbb{N}$. Then, for example, $\left\Vert D^{{\boldsymbol{\gamma}}%
}f\right\Vert _{\infty,U}:=$ $\max_{1\leq i\leq m}\left\Vert
D^{{\boldsymbol{\gamma}}}f_{i}\right\Vert _{\infty,U},$ and
$[D^{{\boldsymbol{\gamma}}}f]_{\alpha,U}:=\max_{1\leq i\leq m}%
[D^{{\boldsymbol{\gamma}}}f_{i}]_{\alpha,U}$ and so on.
\end{remark}

\section{Uniform convergence of kernel regression estimates}

Let $(X,Y),(X_{1},Y_{1}),(X_{2},Y_{2}),\dots$ be independent random vectors in
$\mathbb{R}^{d}\times\mathbb{R}$ with joint density $f_{XY}$ and marginal
density defined by%
\[
f_{X}(x)=\int f_{XY}(x,y)dy.
\]
Letr $\Phi$ be a class of measurable functions $\varphi:\mathbb{R}%
\rightarrow\mathbb{R}$ with $\mathbb{E}[\varphi^{2}(Y)]<\infty$ with $1\in
\Phi$. For any function $\varphi\in\Phi$, $n\in\mathbb{N}$ and bandwidth
$0<\delta\equiv\delta_{n}<1,$ we define the kernel-type estimator%

\[
r_{n,\varphi}(x)=\frac{1}{n\delta^{d}}\sum_{i=1}^{n}\varphi(Y_{i})K\left(
\frac{x-X_{i}}{\delta}\right)
\]
where $K$ is a suitable kernel function. By choosing $\varphi=1$, one
(formally) obtains an estimator for the marginal density $f_{X}.$ This kernel
density estimator, denoted by $r_{n,1},$ is an important special case in the
family of kernel estimators $r_{n,\varphi}$ for $\varphi\in\Phi$. In
particular, for any fixed $\varphi$ the estimate $r_{n,\varphi}$ converges to
\[
r_{\varphi}(x):=\int\varphi(y)f_{XY}(x,y)dy
\]
under suitable conditions that are given below.

Let $K:\mathbb{R}^{d}\rightarrow\mathbb{R}_{+}$ be a measurable kernel satisfying

\begin{itemize}
\item $\Vert K\Vert_{\infty}=K_{\infty}<\infty,$ $\int K(x)\,dx=1$ and $\int
x_{i}K(x)\,dx=0$ for $i=1,\ldots,d;$

\item $K$ has a support contained in $\left[  -\frac{1}{2},\frac{1}{2}\right]
^{d}$;

\item For any fixed $\gamma>0,$ the class $\mathcal{K}=\{x\mapsto
K(\gamma(x-z)):z\in I\}$ is a measurable VC-type class of functions from
$\mathbb{R}^{d}$ to $\mathbb{R}$.
\end{itemize}

Let $\mathsf{S}$ be a compact subset of $\mathbb{R}^{d}$\textsf{ and }assume
that there exists a bounded open set $U\subset\mathbb{R}^{d}$ with
$U\supset\mathsf{S}$ such that $r_{\varphi}\in\mathcal{H}^{1,\alpha}%
(\overline{U})$ for some $0<\alpha\leq1$
with $\mathcal{H}^{1,\alpha}$ as defined in Appendix~\ref{SmoothC}. Obviously there exists $\delta
_{1}>0$ such that%
\[
\left\{  x+z:x\in\mathsf{S}\text{ and }\left\vert z\right\vert \leq\delta
_{1}\right\}  \subset U,
\]
where $\left\vert \cdot\right\vert $ denotes the max-norm in $\mathbb{R}^{d}.$

Let us firstly analyse the bias of $r_{n,\varphi}.$ By our assumptions on $r_\varphi$, the
gradient $\nabla r_{\varphi}$ is uniformly H\"{o}lder continuous of order
$0<\alpha\leq1$ on $U,$ which implies that 
\[
\left\vert \nabla r_{\varphi}(x_{1})-\nabla r_{\varphi}(x_{2})\right\vert \leq
\Vert r_{\varphi}\Vert_{\mathcal{H}%
^{1,\alpha}(\overline{U})}\left\vert x_{1}-x_{2}\right\vert ^{\alpha},\quad x_{1},x_{2}\in U,
\]
%with $\left\vert\cdot\right\vert$ denoting the max-norm on the left, 
see Remark~\ref{remHoel}. For any $x\in\mathsf{S}$ one so has
\begin{align}
|\mathbb{E}[r_{n,\varphi}(x)]-r_{\varphi}(x)| &  =\left\vert \delta^{-d}\int
K\left(  \frac{x-u}{\delta}\right)  du\int\varphi(y)f_{XY}%
(u,y)\,\,dy-r_{\varphi}(x)\right\vert\notag \\
&  =\left\vert \int K(-u)\left(  r_{\varphi}(x+\delta u)-r_{\varphi
}(x)\right)  du\right\vert\notag  \\
&  =\left\vert \int K(u)du\int_{0}^{\delta}u^{\top}\left(  \nabla r_{\varphi
}(x+tu)-\nabla r_{\varphi}(x)\right)  \,dt\right\vert \notag \\
&  \leq\int K(u)du\int_{0}^{\delta}\left\vert u^{\top}\left(  \nabla
r_{\varphi}(x+tu)-\nabla r_{\varphi}(x)\right)  \right\vert \,dt \notag \\
&  \leq {d}\Vert r_{\varphi}\Vert_{\mathcal{H}%
^{1,\alpha}(\overline{U})}(\delta/2)^{\alpha+1}\text{ \ \ for }\delta\leq\delta_{1}.\label{det_term_err}
\end{align}
We now turn to the stochastic part of the error.
The next result can be extracted from Section~3 of \cite{dony2009uniform}. 
\begin{theorem}
Let $\Phi$ be a VC-type class of functions with envelope function
$F(y):=\sup_{\varphi\in\Phi}\varphi(y).$ Suppose that 
\[
\mu_{2}=\sup_{x\in \mathsf{S}}\mathbb{E}\left[  F^{2}(Y)|X=x\right]  <\infty.
\]
Then it holds
\[
\mathbb{E}\Bigl[\sup_{\varphi\in\Phi}\left\Vert r_{n,\varphi}-\mathbb{E}%
[r_{n,\varphi}]\right\Vert _{\mathsf{S}}\Bigr]\leq\frac{C}{\delta^{d}\sqrt{n}}%
\sqrt{\mathbb{E}\left[  \left(  G(X,Y)\right)  ^{2}\right]  }%
\]
with $G(x,y)=\sup_{\varphi\in\Phi}\sup_{z\in \mathsf{S}}g_{\varphi,z}(x,y)$ and some
constant $C>0$ where $g_{\varphi,z}(x,y)=\varphi(y)K((z-x)/\delta).$
\end{theorem}

\begin{corollary}
\label{cor:concentr} Note that it holds
\begin{align*}
\mathbb{E}\left[  \left(  G(X,Y)\right)  ^{2}\right]   &  =\mathbb{E}\left[
\sup_{\varphi\in\Phi}\sup_{z\in \mathsf{S}}\varphi^{2}(Y)K^{2}\left(  \frac{z-X}%
{\delta}\right)  \right]  \\
&  \leq\delta^{d}\kappa^{2}\int_{-\infty}^{\infty}\mathbb{E}\left[
F^{2}(Y)\mid X=z-u\delta\right]  f_{X}(z-uh)\,du\\
&  \leq\delta^{d}\kappa^{2}\Vert f_{X}\Vert_{\mathsf{S}}\mu_{2}.
\end{align*}
Hence
\[
\mathbb{E}\Bigl[\sup_{\varphi\in\Phi}\left\Vert r_{n,\varphi}-\mathbb{E}%
[r_{n,\varphi}]\right\Vert _{\mathsf{S}}\Bigr]\leq\frac{C}{\sqrt{n\delta^{d}}}%
\sqrt{\kappa^{2}\Vert f_{X}\Vert_{\mathsf{S}}\mu_{2}}.
\]

\end{corollary}

\bibliography{schroe}
\bibliographystyle{plain}

\end{document}